\documentclass[11pt,letterpaper,twocolumn]{article}

\usepackage[margin=0.75in]{geometry}

\usepackage[utf8]{inputenc}
\usepackage[T1]{fontenc}

\usepackage[numbers,sort&compress]{natbib}
\usepackage{hyperref}
\usepackage{url}
\usepackage{booktabs}
\usepackage{amsfonts}
\usepackage{amsmath}
\usepackage{amsthm}
\usepackage{newtxtext}
\usepackage{newtxmath}
\newtheorem{definition}{Definition}
\newtheorem{proposition}[definition]{Proposition}
\newtheorem{corollary}[definition]{Corollary}
\usepackage{nicefrac}
\usepackage{microtype}
\usepackage{graphicx}
\usepackage{xcolor}
\usepackage{multirow}
\usepackage{array}
\usepackage{subcaption}
\usepackage{enumitem}
\usepackage{xspace}
\usepackage{tabularx}
\usepackage[section]{placeins}
\usepackage{colortbl}
\usepackage{setspace}
\usepackage[ruled,vlined,linesnumbered]{algorithm2e}
\usepackage{placeins}
\usepackage{titling}

\newcommand{\epikg}{EpiKG\xspace}

\newcommand{\eg}{e.g.,\xspace}

\newcommand{\pp}{\,\text{pp}}
\newcommand{\ci}[2]{[#1, #2]}
\newcommand{\up}[1]{\textcolor{green!60!black}{$\uparrow$#1\pp}}
\newcommand{\dn}[1]{\textcolor{red!70!black}{$\downarrow$#1\pp}}
\newcommand{\smark}{\textcolor{green!60!black}{\checkmark}}
\newcommand{\xmark}{\textcolor{red!80!black}{$\times$}}

\hypersetup{
  colorlinks=true,
  linkcolor=blue!60!black,
  citecolor=blue!60!black,
  urlcolor=blue!60!black,
  pdfborder={0 0 0},
  pdftitle={ClinicalBench: Stress-Testing Assertion-Aware Retrieval for Cross-Admission Clinical QA on MIMIC-IV},
  pdfauthor={Alex Stinard, MD}
}

\title{ClinicalBench: Stress-Testing Assertion-Aware Retrieval\\for Cross-Admission Clinical QA on MIMIC-IV}
\author{Alex Stinard, MD\\
  Department of Clinical Sciences, College of Medicine\\
  University of Central Florida, Orlando, FL 32816\\
  \texttt{alex.stinard@ucf.edu}}
\date{\today}

\begin{document}

\twocolumn[%
  \begin{center}
    {\LARGE\bfseries ClinicalBench: Stress-Testing Assertion-Aware Retrieval\par}
    \vspace{0.25em}
    {\LARGE\bfseries for Cross-Admission Clinical QA on MIMIC-IV\par}
    \vspace{0.9em}
    {\large Alex Stinard, MD\par}
    \vspace{0.25em}
    {\small Department of Clinical Sciences, College of Medicine\par}
    {\small University of Central Florida, Orlando, FL 32816\par}
    {\small \texttt{alex.stinard@ucf.edu}\par}
    \vspace{0.6em}
    {\small \today\par}
    \vspace{0.8em}
    \rule{\linewidth}{0.4pt}\par
    \vspace{0.3em}
    {\textsc{Preprint --- arXiv version}\par}
    \vspace{0.2em}
    \rule{\linewidth}{0.4pt}
  \end{center}
  \vspace{0.6em}
  \begin{minipage}{\linewidth}
    \begin{center}\textbf{\large Abstract}\end{center}
    \small

\noindent\textbf{Objective.}
Reasoning benchmarks measure clinical performance on clean inputs. We evaluate the step before reasoning: retrieval over real EHR notes, where negation, temporality, and family-versus-patient attribution can flip a correct answer to a wrong one.

\noindent\textbf{Materials and Methods.}
EpiKG carries an assertion label and a temporality tag with every fact in a patient knowledge graph, then routes retrieval by question intent. ClinicalBench is a 400-question test over 43 MIMIC-IV patients across 9 assertion-sensitive categories. A 7-condition ablation tests each piece of EpiKG across six LLMs (Opus 4.6, GPT-OSS 20B, MedGemma 27B, Gemma~4 31B, MedGemma 1.5 4B, Qwen 3.5 35B). Three physicians blindly adjudicated 100 paired items.

\noindent\textbf{Results.}
Author-blind primary endpoint: leave-author-out paired exact McNemar on 50 items (Hird $\times$ Nadeem unanimous strict), $\Delta = +22.0\pp$ \ci{+5.1\pp}{+31.5\pp} (95\% Newcombe CI), $p=0.0192$. The architectural novelty is C2b (Contriever dense-RAG) $\to$ C4g\_kw (intent-aware KG-RAG) on the change-excluded $n=362$ endpoint: $+8.84\pp$ (paired McNemar $p=1.79\times10^{-3}$); $+12.43\pp$ under oracle intent. Sensitivity analyses: three-rater physician majority $+24.0\pp$ ($p=0.0075$; Fleiss' $\kappa=0.413$; subject to single-author circularity since the author is R1); deterministic keyword proxy $+39.5\pp$ over LLM-alone (reproducibility tool, not a clinical correctness claim). The audit found 56\% of auto-generated references defective.

\noindent\textbf{Discussion.}
Across the six models, the gain shrinks as the LLM-alone baseline rises ($\beta=-1.123$, $r=-0.921$, $p=0.009$). With $n=6$ this looks more like regression to the mean than encoding substituting for model size. The author built the system, generated the initial gold standard, and performed the internal audit. The primary endpoint uses external physician ratings with the author left out.

\noindent\textbf{Conclusion.}
Carrying assertion labels and routing by question intent improve cross-admission clinical QA across six LLMs. ClinicalBench and the evaluation artifacts are public.

  \end{minipage}
  \vspace{1.2em}
]

\section{Background and Significance}
\label{sec:intro}

Large language models match or exceed physician-level performance on medical licensing exams~\cite{singhal2023medpalm2, saab2024medgemini, tu2024amie}, and reasoning benchmarks like HealthBench Professional, MedQA, and USMLE-style items measure that last mile of clinical reasoning given clean vignettes. Real EHR use exposes a complementary, undermeasured layer: \emph{retrieval faithfulness} on messy charts, where negation, temporal drift, source conflict, and semantic compression must be navigated before reasoning. The harder question is not whether AI can reason like a physician but whether it can read like one---physicians, of course, do both. A single sentence---``patient denies chest pain, sister had MI at 45, will consider statin if lipids remain elevated''---encodes negation, family attribution, hypothetical intent, and an implicit present condition. Clinical NLP detects these assertions accurately~\cite{uzuner2011i2b2assertion, gul2025beyondnegation}, but RAG pipelines flatten the context, conflating ``patient denies'' with ``patient has.'' This is the \emph{epistemic propagation gap}, inside a broader structural-representation gap---assertion typing, temporal indexing, experiencer attribution preserved to retrieval---that reasoning benchmarks do not probe.

\begin{sloppypar}
To the best of current knowledge, no patient-level clinical KG-RAG system jointly preserves assertion state on graph edges and routes retrieval by question intent.
OMOP excludes negated conditions from \texttt{CONDITION\_OCCURRENCE}~\cite{omop2024}, and FHIR provides \texttt{verificationStatus} for \texttt{Condition} resources only.
Existing graph-augmented RAG systems---including GraphRAG~\cite{edge2024graphrag}, GFM-RAG~\cite{luo2025gfmrag}, KARE~\cite{jiang2025kare}, and Medical-Graph-RAG~\cite{wu2025medicalgraphrag}---build KGs that discard the metadata distinguishing ``patient has diabetes'' from ``rule out diabetes.''
A parallel \emph{temporal integration gap} exists: clinical events admit bi-temporal storage (valid + transaction time, in the Snodgrass tradition; cf.\ Zep~\cite{rasmussen2025zep}) plus an NLP-asserted temporality label $\tau_a \in \{\textsc{Past}, \textsc{Current}, \textsc{Future}\}$, yet existing systems model at most a subset~\cite{huang2024medtkg,rasmussen2025zep}.
\end{sloppypar}

The core empirical finding is interactional: assertion preservation alone does not improve aggregate accuracy unless retrieval is also routed by question type. Three contributions are made:

\begin{enumerate}[leftmargin=*, topsep=4pt, itemsep=2pt]
  \item \textbf{ClinicalBench.} A 400-question single-site, same-record stress test over 43 MIMIC-IV patients (convenience sample; 32 with two admissions, 11 single-admission) and 9 assertion-sensitive categories, exposing category$\times$condition interactions aggregate scores hide. It targets retrieval faithfulness on real charts rather than reasoning on clean vignettes, complementing exam-style benchmarks at a different layer. SliceBench, a small supporting case study on record complexity, is also introduced.

  \item \textbf{\epikg and the epistemic propagation gap.} The loss of assertion metadata across clinical NLP pipelines is formalized, an information-theoretic loss bound is derived (Section~\ref{sec:system:formal}, Appendix~\ref{app:formal}), and a patient-level clinical KG-RAG system is implemented that preserves assertion and temporal metadata while routing retrieval by question intent.

  \item \textbf{Author-blind primary endpoint and architectural novelty.} The author-blind primary is a paired test: leave-author-out exact McNemar on $n=50$ unanimous-strict items adjudicated by two external physicians, yielding $\Delta=+22.0\pp$ (95\% Newcombe CI $[+5.1, +31.5]$, $p=0.0192$). The architectural novelty is the paired delta of intent-aware KG-RAG over a strong dense-RAG baseline (Contriever), C2b$\to$C4g\_kw $+8.84\pp$ on the change-excluded $n=362$ endpoint (McNemar $p=1.79\times10^{-3}$; oracle $+12.43\pp$). Secondary sensitivities are demoted: three-rater majority $+24.0\pp$ (single-author circularity since the author is one rater) and a deterministic reproducibility proxy (keyword evaluator) $+39.5\pp$ (not a clinical-correctness claim). Cross-model convergence across $n=6$ models is descriptive only: a linear regression of C1 baseline against C1$\to$C4g\_oracle delta yields $\beta=-1.123$, $r=-0.921$, $p=0.009$, consistent with regression to the mean rather than encoding substituting for parameter count.
\end{enumerate}

\noindent The author designed the benchmark, built the system, and conducted the internal evaluation; this circularity is structurally mitigated by frozen evaluation artifacts, external physician evaluation, and cross-model replication, but readers should weight claims accordingly (Section~\ref{sec:results:circularity}).

\noindent Together these yield a benchmark-supported design hypothesis: preserve epistemic metadata, route retrieval by intent, and evaluate by category$\times$condition interaction rather than aggregate score. The central research question---\emph{when does structured epistemic context help, hurt, or break even?}---is answered interactionally on a single-site, in-distribution stress test designed for retrieval faithfulness rather than cross-site generalization.

\subsection{Related Work}
\label{sec:related}

Prior work is organized along four axes (extended discussion and Table~\ref{tab:gap} in Appendix~\ref{app:related_extended}).

\paragraph{Clinical reasoning benchmarks.}
Reasoning evaluations are complementary. HealthBench Professional~\cite{healthbench2026}, MedQA~\cite{jin2021medqa}, and MedPaLM~2~\cite{singhal2023medpalm2} score reasoning on vignettes with facts pre-supplied; \epikg measures retrieval faithfulness on real longitudinal EHRs with dispersed facts and negation, temporal, and source ambiguity. The two probe different stages: the last mile (reasoning given clean inputs) versus the first mile (reading the right patient from messy charts).

\paragraph{Medical RAG and clinical QA.}
Graph-augmented retrieval is a leading paradigm: GraphRAG~\cite{edge2024graphrag}, GFM-RAG~\cite{luo2025gfmrag}, and Medical-Graph-RAG~\cite{wu2025medicalgraphrag} build population-level graphs but do not propagate note-derived assertion or temporal metadata (bi-temporal storage with NLP-asserted scope label, in our framing). Existing benchmarks---MedPaLM~2~\cite{singhal2023medpalm2}, MIRAGE~\cite{xiong2024mirage}, emrQA~\cite{pampari2018emrqa}---target factual recall or grounded retrieval, not assertion-faithful longitudinal QA over real EHRs.

\paragraph{Clinical KG construction.}
Multi-LLM KG-RAG~\cite{chen2026multilllmkgrag}, AutoRD~\cite{li2024autord}, and RECAP-KG~\cite{remy2024recapkg} apply LLMs to clinical KG construction but do not propagate assertion status into the final graph.

\paragraph{Assertion detection and temporal KGs.}
NegEx~\cite{chapman2001negex}, ConText~\cite{harkema2009context}, and Gul et al.~\cite{gul2025beyondnegation} treat assertion detection as terminal annotation; MedTKG~\cite{huang2024medtkg} and Graphiti~\cite{rasmussen2025zep} implement temporal KGs but lack epistemic propagation. Two structural gaps emerge: an \emph{epistemic propagation gap} (assertion labels are not persisted into KGs) and a \emph{temporal integration gap} (temporal formalisms are annotation layers, not retrieval-participating edge attributes). \epikg closes both by carrying assertion and temporal metadata as first-class properties through every pipeline stage.

\section{Objective}
\label{sec:objective}
To evaluate whether preserving assertion and temporal metadata in a patient-level clinical knowledge graph, then routing retrieval by question intent, improves cross-admission clinical question answering over electronic health records.

\section{Materials and Methods}
\label{sec:methods}
\label{sec:system}

\subsection{Method Overview}
\label{sec:system:overview}

\epikg implements three ideas (Figure~\ref{fig:architecture}): (1)~end-to-end epistemic preservation, carrying assertion labels through extraction, OMOP mapping, KG materialization, and retrieval; (2)~bi-temporal edge storage (valid time, transaction time; in the Snodgrass tradition, cf.\ Graphiti~\cite{rasmussen2025zep}) plus an NLP-asserted temporality label $\tau_a \in \{\text{Past, Current, Future}\}$ derived from clinical-text scope (data-modeling clarification in Appendix~\ref{app:temporal_model}); and (3)~intent-aware routing matching graph traversal to question type. The first two are infrastructure; the third is where the performance gain originates.

\begin{figure*}[!htbp]
\centering
\includegraphics[width=\textwidth]{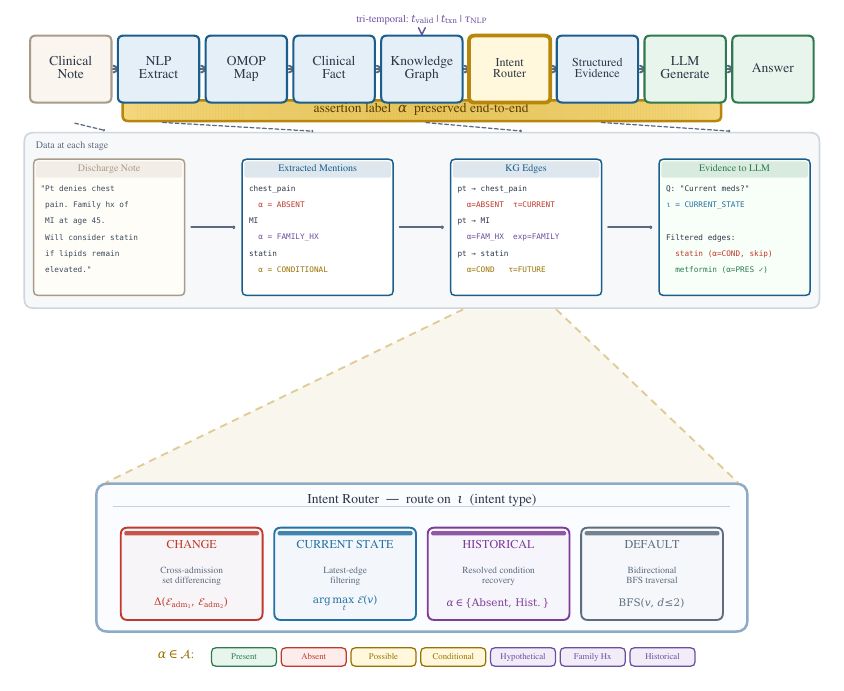}
\caption{\epikg system workflow with concrete data examples. \textbf{Top}: 9-stage pipeline from clinical note to answer, with the gold $\alpha$ ribbon tracing assertion preservation end-to-end. \textbf{Middle}: actual data at each stage---a discharge note with negation, family history, and conditional language is extracted into assertion-labeled mentions, materialized as KG edges with temporality, and filtered by intent-aware routing. \textbf{Bottom}: the four routing strategies with formal operations. The example shows how a \textsc{Current\_State} query filters out conditional edges while preserving confirmed medications.}
\par\small\textit{Alt text: Multi-row workflow diagram. The top row shows clinical note ingestion through extraction, OMOP mapping, graph construction, retrieval, and answer generation. A highlighted assertion label is preserved across stages. Middle panels show example note text, extracted mentions, graph edges, and routed evidence. The bottom row compares routing operations for default, change, current-state, and historical queries.}
\label{fig:architecture}
\end{figure*}

\subsection{Epistemic Assertion Schema}
\label{sec:system:assertion}

Clinical notes contain qualified statements (``\textit{no evidence of pneumonia},'' ``\textit{possible CHF},'' ``\textit{mother had breast cancer}'') that standard representations discard: OMOP excludes negated conditions~\cite{omop2024}; FHIR limits assertion metadata. \epikg defines a seven-value assertion taxonomy:
\begin{equation}
\label{eq:assertion}
\alpha \in \{\text{\scriptsize Pres., Abs., Poss., Cond., Hypo., Fam.Hx., Hist.}\}
\end{equation}
extending the i2b2 six-class taxonomy~\cite{uzuner2011i2b2assertion} by separating \textsc{Historical} from \textsc{Family\_History} (Appendix~\ref{app:assertions}). A rule-based classifier (122 scope-aware trigger patterns) assigns $\alpha$ with a confidence score, propagated through every stage. Each edge carries bi-temporal metadata (valid + transaction time) plus an NLP-asserted temporality label $\tau_a$, with Allen-style interval relations stored as edge metadata (data-modeling clarification in Appendix~\ref{app:temporal_model}).

\subsection{Formal Epistemic Preservation}
\label{sec:system:formal}

The epistemic invariant is formalized as a testable pipeline property (Appendix~\ref{app:formal}). An assertion-blind pipeline collapses all labels to \textsc{Present}, reducing assertion entropy to zero~\cite{shannon1948}; its faithfulness bound is $1 - f_{\textit{np}}(c)$, where $f_{\textit{np}}$ is the fraction of non-present mentions---substantially below~1 for concepts like pneumonia or diabetes. Empirical consequences are measured via category-stratified accuracy in Section~\ref{sec:results}.

\subsection{Intent-Aware Retrieval (C4g)}
\label{sec:system:intentaware}

The base retrieval pipeline uses bidirectional BFS over patient KG edges and OMOP vocabulary relationships (Appendix~\ref{app:retrieval_details}), but treats all questions uniformly. Different clinical question types require fundamentally different graph operations: \emph{change} requires cross-admission set differencing, \emph{current-state} needs the most recent valid edges, \emph{historical} must recover resolved conditions.

\paragraph{Intent classifier.}
A rule-based classifier maps each question to \textsc{Change}, \textsc{Current\_State}, \textsc{Historical}, or \textsc{Default} (Algorithm~\ref{alg:c4g}, Appendix~\ref{app:routing}). Primary results use keyword-only classification (production-realistic); oracle classification with category metadata is an upper bound. Keyword classification reduces Opus C4g from 68.5\% to 60.2\% ($-8.3\pp$; Section~\ref{sec:results:intent_sensitivity}).

\paragraph{Routing strategies.}
\textsc{Change} partitions edges by \texttt{hadm\_id} and computes set differences across admission pairs. \textsc{Current\_State} filters edges to $\tau_a = \textsc{Current}$ or open validity. \textsc{Historical} selects $\tau_a = \textsc{Past}$ edges, augmented by admission-based inference (concepts in earlier but not the latest admission are labeled ``resolved''). Each intent triggers a type-specific prompt template (Appendix~\ref{app:example}). Figure~\ref{fig:worked_example} shows a historical question answered incorrectly under C1 (no evidence) and C4 (stale PRESENT label) but correctly under C4g's temporal filtering.

\begin{figure*}[!htbp]
\centering
\includegraphics[width=\textwidth]{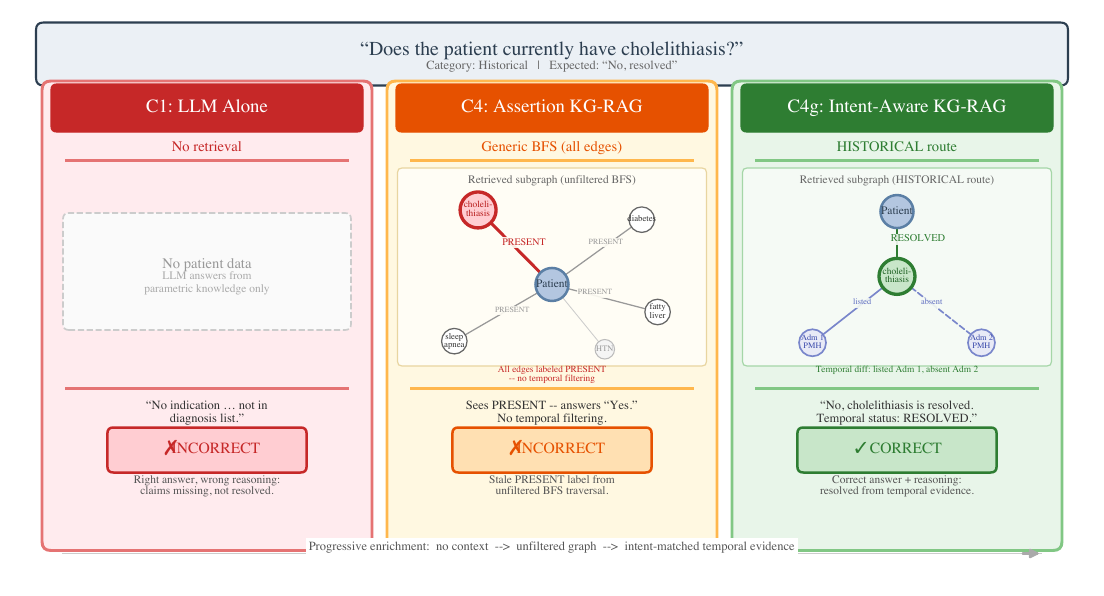}
\caption{Worked example with retrieved knowledge subgraphs. A \textsc{Historical} question is answered under three conditions. C1 (left) has no patient data. C4 (center) retrieves an unfiltered BFS subgraph where all edges are labeled \textsc{Present}---including a stale label for cholelithiasis (red). C4g (right) applies \textsc{Historical} routing, filtering to resolved status with temporal evidence from two admissions. Mini-graph insets show the actual retrieved subgraph structure.}
\par\small\textit{Alt text: Three side-by-side answer cards for one historical question. The left C1 card lacks patient evidence, the center C4 card retrieves a stale present edge, and the right C4g card filters by historical intent to recover resolved-status evidence across two admissions.}
\label{fig:worked_example}
\end{figure*}

\label{sec:benchmarks}

Existing clinical QA benchmarks evaluate factual recall (MedQA~\cite{jin2021medqa}) or agent task completion~\cite{medagentbench2025} but do not test epistemic qualifiers or cross-admission reasoning. ClinicalBench\footnote{Unrelated to identically named benchmarks in other clinical NLP subfields.} is a same-record retrieval-faithfulness stress test on MIMIC-IV~\cite{johnson2023mimiciv}; SliceBench is a small supporting case study.

\subsection{ClinicalBench: Assertion-Sensitive Clinical QA}
\label{sec:bench:clinicalbench}

ClinicalBench comprises 400 questions over 43 MIMIC-IV patients across two tasks (\textbf{A}: 200 negation-aware retrieval; \textbf{B}: 200 temporal reasoning) and 9 categories (\textit{negation}, \textit{conditional}, \textit{uncertainty}, \textit{family\_history}, \textit{sequence}, \textit{current\_state}, \textit{duration}, \textit{historical}, \textit{change}). The v2 reference set began as auto-generated NLP labels with 54 physician corrections and is provisional. Questions are authored from the same de-identified charts later ingested---an in-distribution retrieval-faithfulness stress test, not cross-site generalization. Four ablation conditions progressively add components; two bookends (C6 all-notes, C7 deterministic-KG) probe design boundaries (Table~\ref{tab:clinicalbench_conditions}).

\paragraph{Cohort structure.}
The cohort comprises 43 MIMIC-IV patients selected as a convenience sample (no formal stratification or random sampling); 11 patients have a single hospital admission, 32 have two admissions, and 0 have three or more. The `cross-admission' framing applies cleanly to the 32-patient two-admission subset; on the 11-patient single-admission subset, retrieval still operates over multiple notes within an admission. Demographics in Appendix~\ref{app:demographics}.

\paragraph{Linguistic categories, not computable phenotypes.}
ClinicalBench's 9 categories (negation, conditional, uncertainty, family\_history, sequence, current\_state, duration, historical, change) are linguistic constructs adapted from the i2b2 assertion taxonomy~\cite{uzuner2011i2b2assertion}. They are NOT computable phenotypes in the PheKB / eMERGE / OHDSI sense; the per-category C1$\to$C4g deltas should not be interpreted as phenotype-validation evidence. Phenotype-validatable evaluation requires OMOP/SNOMED concept-sets, multi-site PPV, and chart-review-validated cohorts (Newton et al., \emph{JAMIA} 2013; Hripcsak \& Ryan, \emph{JBI} 2019), none of which this work provides.

\paragraph{Experiencer-attribution defects.}
Post-hoc verification identified 8 items (qids in Appendix~\ref{app:bugged_items}) with experiencer-attribution defects: source \texttt{section=Family History} but gold \texttt{expected\_answer} asserts the disease as a current/historical condition of the patient. These are excluded from the change-excluded keyword endpoint ($n=362$, now reported as a sensitivity comparator; see Endpoints below), reducing the keyword delta from $+40.0\pp$ to $+39.5\pp$. The items remain in the released v2 gold for transparency; v3 corrections are planned post-publication.

\begin{table*}[t]
\centering
\footnotesize
\caption{ClinicalBench conditions. C1--C4g: ablation ladder; C6/C7: bookend baselines; C1b/C4g+: extensions ($n=240$).}
\label{tab:clinicalbench_conditions}
\setlength{\tabcolsep}{4pt}
\begin{tabular}{@{}cllccc@{}}
\toprule
ID & Short Name & Condition & Retrieval & Assertion & Temporal \\
\midrule
C1 & \textsc{LLM-alone} & LLM Alone & None & None & None \\
C2 & \textsc{TF-IDF RAG} & + Vanilla RAG (TF-IDF) & TF-IDF doc chunks & None & None \\
C2b & \textsc{Dense RAG} & + Vanilla RAG (dense) & Contriever doc chunks & None & None \\
C3 & \textsc{KG-RAG} & + KG-RAG (no assertions) & Graph + Doc & None & None \\
C4 & \textsc{KG-RAG+Assert} & + Epistemic KG-RAG & Graph + Doc & Full (7-class) & Bi-temporal+label \\
C4g & \textsc{KG-RAG+Route} & + Intent-Aware KG-RAG & Graph + Doc (type-specific) & Full (7-class) & Bi-temporal+label \\
\midrule
C6 & \textsc{Long Context} & Long Context & All notes & None & None \\
C7 & \textsc{Deterministic KG} & Deterministic KG & KG lookup (no LLM) & Full & Bi-temporal+label \\
\midrule
C1b & \textsc{Discharge Only} & Discharge Summary & Discharge doc only & None & None \\
C4g+ & \textsc{KG-RAG+Notes} & KG-RAG + Full Notes & Graph + Doc + All notes & Full (7-class) & Bi-temporal+label \\
\bottomrule
\end{tabular}
\end{table*}

\paragraph{Endpoints.}
The \textbf{primary endpoint} is the leave-author-out paired exact McNemar test on $n=50$ matched (C1, C4g) qid pairs from the three-rater external adjudication, restricted to Hird $\times$ Nadeem unanimous strict ratings (Section~\ref{sec:results:primary_endpoint}). This is the substantive author-blind comparison; it eliminates single-author circularity at the cost of $n=362 \to n=50$ statistical power. \textbf{Secondary / sensitivity}: (i) deterministic keyword reproducibility proxy on the change-excluded $n=362$ subset (C4g\_keyword vs.\ C1; reported with patient-level cluster bootstrap CIs over 43 patients --- not a clinical-correctness claim); (ii) three-rater majority vote ($n=100$, subject to single-author circularity); (iii) oracle C4g upper bound; (iv) hard cross-admission subset (change $\cup$ current\_state $\cup$ historical, $n=122$, post-selection-inference caveat); (v) C1b vs.\ C4g+ extension ($n=240$); (vi) full $n=400$. \textbf{Diagnostic}: per-category C1$\to$C4g deltas and the C3$\to$C4$\to$C4g decomposition. The \textit{change} category is excluded from secondary (i) due to known reference defects (Section~\ref{sec:results:physician}). Cohort demographics in Appendix~\ref{app:demographics}.

\paragraph{Endpoint pre-registration and provenance.}
The leave-author-out paired exact McNemar statistic ($b=4$, $c=15$, $p=0.0192$) was computed and reported as a sensitivity analysis in commit \texttt{0c510d7} (v85.1, 2026-04-26), prior to the JAMIA-10 external review; its underlying three-rater adjudication data were frozen earlier (commit \texttt{7b388db}, v74) and were not modified subsequently. \textbf{Promotion to primary endpoint occurred in v86 (2026-04-27) post-hoc, in response to external-reviewer feedback identifying single-author circularity as the dominant threat to the originally-declared $n=362$ keyword primary} (commit \texttt{96746e9}, v82, 2026-04-26). Because both endpoints are deterministic computations over already-frozen data, the promotion does not involve additional data collection or model re-running; it does, however, change the headline claim and the statistical-power profile, and we therefore report the leave-author-out primary alongside the originally-declared $n=362$ keyword endpoint as a sensitivity. The 8-item experiencer-attribution exclusion ($n=370 \to 362$) was identified post-hoc during external review (the $+0.5\pp$ impact is documented in \S\ref{sec:results:clinicalbench}). The hard cross-admission $n=122$ subset was selected post-hoc with a post-selection-inference caveat in \S results. Pre-registration was not performed on AsPredicted or OSF; future versions of this benchmark will pre-register endpoints prior to data collection.

\subsection{SliceBench: Complexity-Stratified Case Study}
\label{sec:bench:slicebench}

SliceBench is a small case study (6 MIMIC-IV patients, 144 questions, three complexity tiers) testing whether KG-augmented retrieval scales with record complexity. Five conditions (B0--B4) form a monotone context progression; the critical comparison is B2$\to$B3, which adds structured KG context with assertion metadata while holding documents fixed (Appendix~\ref{app:slicebench_conditions}).

\subsection{Evaluation Protocol}
\label{sec:bench:protocol}

\paragraph{ClinicalBench} uses a \emph{deterministic keyword evaluator} (v2)---exact word-boundary matching with abstention-detection gate (Appendix~\ref{app:reproducibility}). The physician-audited subset provides the most credible human accuracy estimate. Primary answering model: Claude Opus 4.6; cross-model: MedGemma 27B, GPT-OSS 20B, Qwen3.5 35B, Gemma~4 31B, MedGemma 1.5 4B.

\paragraph{SliceBench} uses LLM-as-judge with separated answering/judging models~\cite{xiong2024mirage} (Claude Sonnet 4.5 answers, Opus 4.6 judges).

\paragraph{Statistical reporting.}
The author-blind primary endpoint (leave-author-out paired exact McNemar, $n=50$) is reported with two-sided exact binomial $p$-values and a 95\% Newcombe CI on the paired difference; this endpoint does not depend on any bootstrap. For the keyword sensitivity endpoint and other secondary contrasts we report BCa bootstrap 95\% CIs ($n=2{,}000$, seed 42)~\cite{efron1993bootstrap} with patient-level cluster bootstrap (43 patients) primary and question-level secondary (caveat: $n=43$ clusters is at the low end of cluster-bootstrap reliability; cf.\ Cameron \& Miller~\cite{cameron2015cluster}). McNemar's test~\cite{mcnemar1947note} with Benjamini--Hochberg FDR correction is used for paired condition comparisons and cross-model contrasts. Safety score with asymmetric weighting ($w=2.0$) in Appendix~\ref{app:safety}.

\subsection{Physician Adjudication Protocol}
\label{sec:bench:physician}

The author (board-certified emergency physician, system designer) conducted a blinded internal audit of 120 paired C1/C4g questions with randomized A/B labels, rating each on five dimensions: reference correctness, model correctness, score fairness, safety, utility (Appendix~\ref{app:physician_protocol}). The adjudication supports C4g$>$C1 ($+35.0\pp$ strict, $+31.7\pp$ lenient; paired exact McNemar $p<10^{-8}$ strict) and revealed a 56\% reference-answer defect rate.

\paragraph{External physician adjudication.}
Two independent physicians (senior attending, 20+ years; resident) completed the same blinded 100-item protocol, yielding a three-rater majority vote with Fleiss' $\kappa$ and exact-binomial McNemar $p$-values (Section~\ref{sec:results:external}). This is in-distribution physician adjudication, not multi-site phenotype validation.

\section{Results}
\label{sec:results}

\epikg is evaluated with ClinicalBench (same-record retrieval-faithfulness stress test) and SliceBench (small case study on patient complexity). ClinicalBench full-set results use a deterministic keyword proxy; the physician-adjudicated subset provides the most credible human accuracy estimate; SliceBench uses LLM-as-judge (Section~\ref{sec:bench:protocol}). Claude Opus 4.6 is the primary answering model; cross-model evaluation spans MedGemma 27B, GPT-OSS 20B, Qwen3.5 35B, Gemma~4 31B, and MedGemma 1.5 4B (Appendix~\ref{app:reproducibility}). Both benchmarks report BCa bootstrap 95\% CIs ($n=2{,}000$, seed 42)~\cite{efron1993bootstrap}.

\subsection{ClinicalBench: Primary Ablation}
\label{sec:results:clinicalbench}

\begin{table*}[t]
\centering
\caption{ClinicalBench full-set \emph{proxy} results (400 questions, Claude Opus 4.6, keyword evaluator v2 with abstention detection). Reproducibility scaffolding only --- not the primary endpoint (see Section~\ref{sec:results:primary_endpoint}). Ablation ladder (C1--C4g) progressively adds components; C4 isolates assertion metadata from intent routing; C6 and C7 are bookend baselines. C4g$_\text{kw}$ (keyword routing, secondary/sensitivity) and C4g$_\text{oracle}$ (oracle routing, upper bound) are shown separately. Sig: * denotes BCa CI excluding zero (caveat: $n=43$ patients is at the low end of cluster-bootstrap reliability; cf.\ Cameron \& Miller~\cite{cameron2015cluster}). Best in \textbf{bold}.}
\label{tab:clinicalbench_ablation}
\small
\begin{tabular}{@{}llcl@{}}
\toprule
& Condition & Accuracy & $\Delta$ vs C1 \\
\midrule
C1 & LLM Alone & 21.8\% & --- \\
C2 & + Vanilla RAG (TF-IDF) & 52.0\% & $+30.2\pp$ * \\
C2b & + Vanilla RAG (dense) & 50.8\% & $+29.0\pp$ * \\
C3 & + KG-RAG (no assertions) & 50.0\% & $+28.2\pp$ * \\
C4 & + Assertions (no routing) & 46.2\% & $+24.5\pp$ * \\
C4g$_\text{kw}$ & + Intent-Aware KG-RAG (keyword) & \textbf{60.2\%} & $+38.5\pp$ * \\
C4g$_\text{oracle}$ & + Intent-Aware KG-RAG (oracle) & 68.5\% & $+46.8\pp$ * \\
\midrule
C6 & Long Context (all notes) & 59.2\% & $+37.5\pp$ * \\
C7 & Deterministic KG (no LLM) & ---$^{\dagger}$ & --- \\
\bottomrule
\end{tabular}
\end{table*}

\begin{figure*}[!htbp]
\centering
\includegraphics[width=\textwidth]{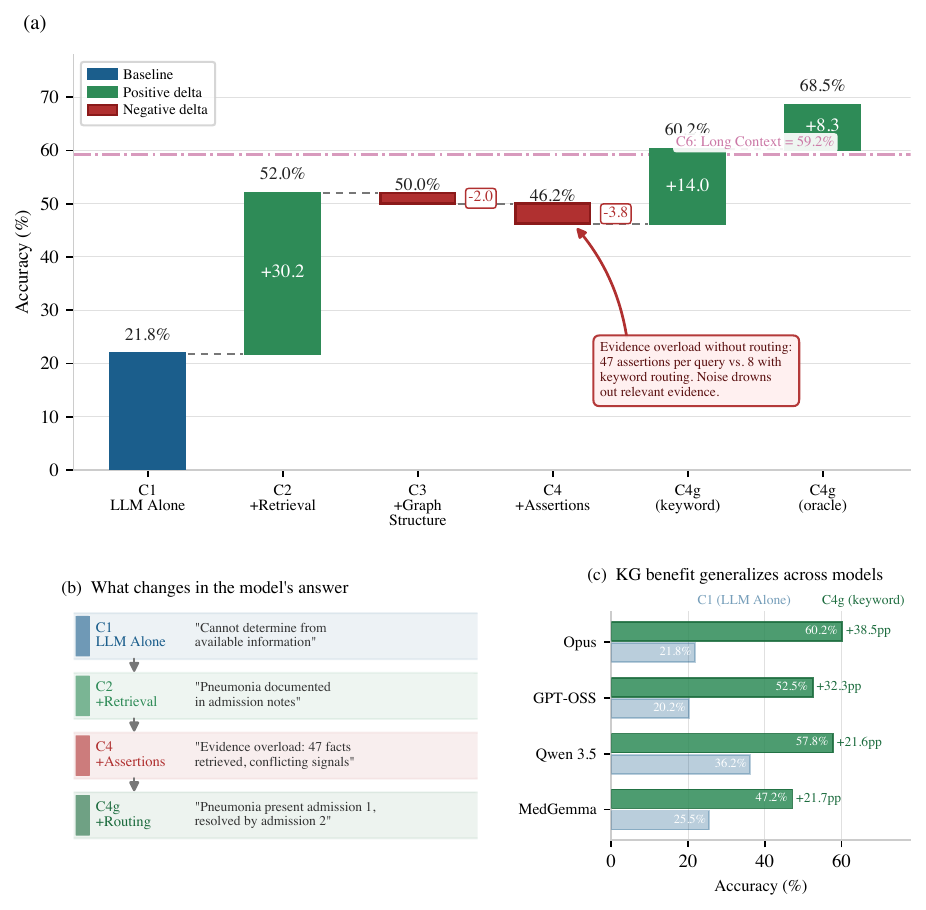}
\caption{Ablation results. \textbf{(a)}~Waterfall chart showing incremental accuracy changes (Opus, $n=400$). Retrieval provides the largest gain ($+30.2\pp$); switching from flat to KG-structured retrieval is neutral ($-2.0\pp$); assertions without routing \emph{hurt} ($-3.8\pp$); keyword routing recovers and extends ($+14.0\pp$); oracle routing adds $+8.3\pp$. \textbf{(b)}~Qualitative progression of model answers across conditions. \textbf{(c)}~KG-RAG benefit generalizes across all six models ($+20$--$47\pp$ over C1).}
\par\small\textit{Alt text: Three-panel ablation figure. A waterfall plot shows accuracy rising with retrieval, falling slightly with assertions alone, and rising again with intent routing. A qualitative answer panel compares condition outputs. A cross-model panel shows KG-RAG gains for every tested model.}
\label{fig:waterfall}
\end{figure*}

ClinicalBench provides three evaluators that yield directionally consistent estimates: physician three-rater majority $+24.0\pp$ ($p = 0.0075$; sensitivity, subject to single-author circularity; Section~\ref{sec:results:external}), internal author adjudication $+35.0\pp$ strict (descriptive, single-rater; Section~\ref{sec:results:physician}), and a deterministic keyword reproducibility proxy $+39.5\pp$ (NOT a clinical correctness claim). The keyword evaluator is reproducibility scaffolding---shallow keyword matching with no polarity check, favoring C4g's structured-answer style---and is reported here for replicability, not as the substantive comparison (Appendix~\ref{app:evaluator_polarity}). On the change-excluded $n=362$ proxy endpoint (8 experiencer-attribution-defective items excluded), keyword C4g reaches 62.4\% versus C1 22.9\% ($+39.5\pp$; McNemar $p = 2.44 \times 10^{-30}$); oracle provides a $+43.1\pp$ upper bound (66.0\%). C6 (long context) scores 59.2\%---$-9.3\pp$ below C4g$_\text{oracle}$ ($p = 0.001$), gap concentrated in current-state (C6 18.0\% vs.\ C4g 70.0\%). $^{\dagger}$C7 returns template refusals on $>$98\% of questions and is semantically 0\% (Appendix~\ref{app:adjudication_full}). The author-blind primary endpoint is the leave-author-out paired exact McNemar reported in Section~\ref{sec:results:primary_endpoint}.

\paragraph{Ablation decomposition.}
C2 (TF-IDF) and C2b (Contriever dense) score comparably (52.0\% vs.\ 50.8\%; $p = 0.62$), and C3 (50.0\%) is similar: retrieval method does not explain the C2$\to$C4g gap. C4 scores 46.2\%---\emph{below} C3 ($-3.8\pp$, n.s.); intent routing recovers and extends ($+14.0\pp$ keyword, $+22.3\pp$ oracle; $p < 10^{-6}$). Per-category, assertions alone help assertion-sensitive (negation $+22.7\pp$, uncertainty $+15.0\pp$) but degrade temporal (historical $-30.0\pp$, sequence $-45.0\pp$); routing reverses these (Appendix~\ref{app:adjudication_full}). The C2b$\to$C4g architectural delta vs.\ dense-RAG baseline is reported separately below.

\paragraph{Intent classification sensitivity.}
\label{sec:results:intent_sensitivity}
Keyword-only C4g (60.2\%) outperforms C4 without routing ($+14.0\pp$) and the best non-KG baseline C2 ($+8.2\pp$), indicating the architecture's value does not require an oracle classifier (per-category classifier accuracy in Appendix Table~\ref{tab:keyword_classifier}).

\begin{figure}[!htbp]
\centering
\includegraphics[width=\columnwidth]{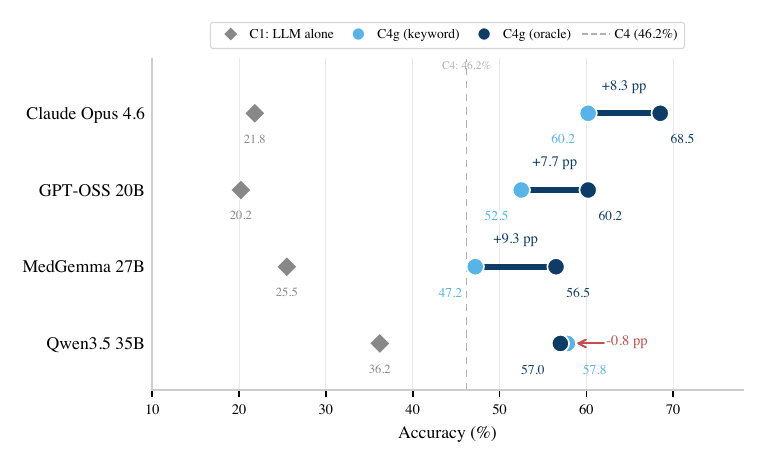}
\caption{Oracle vs.\ keyword-only intent routing across the four models with both routing variants. Dumbbells span from keyword C4g (light blue) to oracle C4g (dark blue); gray diamonds show C1 baselines. The dashed line marks C4 (no routing, 46.2\%). Most models gain modestly from oracle routing; Qwen3.5 (run with \texttt{think:false} due to an Ollama repetition-penalty issue, Appendix~\ref{app:infrastructure_bugs}) does not benefit from oracle routing in this configuration.}
\par\small\textit{Alt text: Dumbbell chart comparing keyword and oracle routing accuracy by model. Each row includes a low C1 baseline point and higher C4g points. Most models improve with oracle routing.}
\label{fig:oracle_keyword}
\end{figure}

\paragraph{Evaluator hierarchy.}
LLM-as-judge ($+28.5\pp$; Appendix~\ref{app:llm_judge}) corroborates the physician and keyword estimates. The keyword evaluator is too strict in 40\% of cases vs.\ too lenient in 5\% (7.5:1 ratio).

\paragraph{Circularity disclosure.}
\label{sec:results:circularity}
The author designed ClinicalBench, built EpiKG, generated initial reference answers, and conducted internal adjudication---a degree of role overlap that could bias results. Three structural mitigations bound this risk: frozen public release of evaluator and predictions; two external physicians confirmed C4g under blinded majority vote ($+24.0\pp$); five additional LLMs all show significant benefit ($+20.4$ to $+43.1\pp$ oracle). Independent replication on a separately authored benchmark is the definitive test.

\paragraph{Hard cross-admission subset.}
On the 122-item subset requiring synthesis across $\geq 2$ admissions, C4g$_\text{oracle}$ reaches 72.1\% versus C1 14.8\% ($+57.4\pp$; 95\% CI: \ci{+47.5\pp}{+66.0\pp}).

\paragraph{Architectural novelty: structured intent-aware retrieval over a strong dense-RAG baseline.}
\label{sec:results:architectural_delta}
C2b (Contriever dense RAG) $\to$ C4g$_\text{kw}$ (intent-aware KG-RAG) on $n=362$ yields $+8.84\pp$ (McNemar $p = 1.79 \times 10^{-3}$); oracle classification yields $+12.43\pp$. On full $n=400$, $+9.50\pp$ keyword / $+17.75\pp$ oracle. This isolates the structural-retrieval-with-routing contribution over a reasonable dense-retrieval baseline; it is the defensible architectural novelty number, separating retrieval-vs-no-retrieval from structured-retrieval-vs-flat-retrieval.

\subsection{Discharge Summary vs.\ KG-RAG (Cross-Model Extension)}
\label{sec:results:extension}

A clinically realistic comparison pits discharge summary alone (C1b) against EpiKG with all notes (C4g+full) on $n=240$ across four models. Three of four models reach significance under cluster bootstrap; GPT-OSS shows a directional but non-significant gain: Opus $+12.5\pp$ (57.5\%$\to$70.0\%), Qwen3.5 $+10.4\pp$ (58.3\%$\to$68.8\%), MedGemma $+8.8\pp$ (55.0\%$\to$63.8\%), GPT-OSS $+1.7\pp$ (60.4\%$\to$62.1\%, $p = 0.32$). The range $+1.7$ to $+12.5\pp$ is consistent with structured retrieval over full clinical notes generalizing across parameter and training differences (per-category breakdowns in Appendix Table~\ref{tab:extension_categories}).

\subsection{Category$\times$Condition Interaction}
\label{sec:results:categories}

\begin{figure*}[!htbp]
\centering
\includegraphics[width=0.85\textwidth]{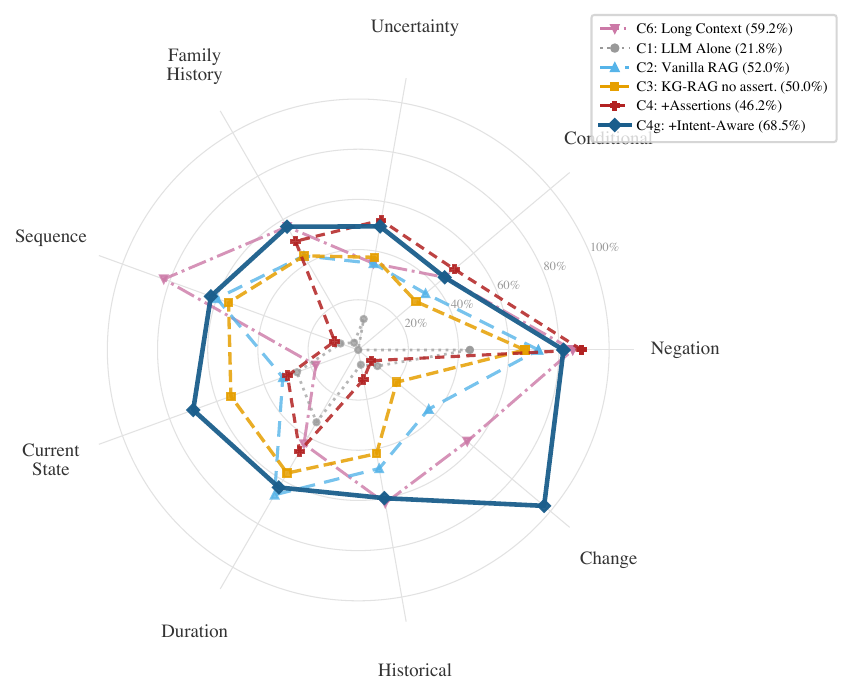}
\caption{ClinicalBench per-category accuracies across six conditions (Claude Opus 4.6, evaluator v2). This plot is descriptive only: category counts vary, no multiplicity correction is applied, and categories with known evaluation defects, especially change, can inflate apparent gaps. C4 is lower than C3 overall, while C4g is highest overall.}
\par\small\textit{Alt text: Radar chart showing per-category accuracy for six ClinicalBench conditions. C4g covers the largest area overall, especially on cross-admission categories. C4 is smaller than C3 in several temporal categories, illustrating the assertion-without-routing regression.}
\label{fig:radar}
\end{figure*}

Per-category accuracy (Figure~\ref{fig:radar}; Appendix Table~\ref{tab:crossmodel_categories}) reveals a \emph{category$\times$condition interaction} (diagnostic only; $n=20$--30 per category, no multiplicity correction). C4g improves all 9 categories over C1, with the largest gains in cross-admission synthesis. C4 helps assertion-sensitive categories but degrades temporal ones, resolved by intent routing (Appendix Figure~\ref{fig:deltas}); C6 lags C4g most on current state ($-52\pp$) and conditional ($-5\pp$), with parity elsewhere.

\subsection{Cross-Model Evaluation}
\label{sec:results:crossmodel}

\begin{table*}[t]
\centering
\caption{ClinicalBench cross-model results (change-excluded keyword sensitivity endpoint, $n=362$, keyword evaluator v2; the substantive primary endpoint is the leave-author-out paired exact McNemar in Section~\ref{sec:results:primary_endpoint}). C4g$_\text{oracle}$ shown as upper bound; all six models benefit (oracle deltas $+20.4$ to $+43.1\pp$, all $q < 10^{-10}$ after BH-FDR).}
\label{tab:crossmodel}
\small
\begin{tabular}{@{}llccl@{}}
\toprule
& Model & C1 & C4g$_\text{oracle}$ & $\Delta$ \\
\midrule
Commercial & Claude Opus 4.6 & 22.93\% & \textbf{66.02\%} & $+43.1\pp$ \\
Open & GPT-OSS 20B & 21.82\% & 59.39\% & $+37.6\pp$ \\
Medical & MedGemma 27B & 27.90\% & 55.80\% & $+27.9\pp$ \\
Open-weights & Gemma~4 31B & 36.74\% & 61.05\% & $+24.3\pp$ \\
Medical (small) & MedGemma 1.5 4B & 35.91\% & 56.35\% & $+20.4\pp$ \\
Reasoning & Qwen3.5 35B & 39.50\% & 60.77\% & $+21.3\pp$ \\
\bottomrule
\end{tabular}
\end{table*}

All six models benefit (Table~\ref{tab:crossmodel}; oracle deltas $+20.4$ to $+43.1\pp$, all $q < 10^{-10}$ after BH-FDR), with benefit inverse to baseline strength (Appendix Table~\ref{tab:crossmodel_categories}).

To assess whether the convergence of C4g$_\text{oracle}$ accuracies (range 55.8--66.0\%) despite C1 spanning 21.8--39.5\% reflects encoding partially substituting for parameter count, we regressed C1$\to$C4g$_\text{oracle}$ delta on C1 baseline. Slope $\beta = -1.123$, Pearson $r = -0.921$, $p = 0.009$ ($n=6$ models). The strong negative slope is consistent with regression to the mean rather than encoding-substitution; the substitution hypothesis is therefore not supported by this evidence and is removed from the contribution claims (\S\ref{sec:intro}).

\subsection{SliceBench}
\label{sec:results:slicebench}

A small supporting case study (SliceBench, 6 patients, 144 questions) is consistent with a complexity-dependent KG effect but does not reach aggregate significance (Appendix~\ref{app:slicebench}).

\subsection{Physician Adjudication}
\label{sec:results:physician}

A blinded internal audit of 120 paired questions (Section~\ref{sec:bench:physician}, Appendix~\ref{app:adjudication_full}) shows C4g at 62.5\% strict / 84.2\% lenient vs.\ C1 at 27.5\% / 52.5\%. Deltas: $+35.0\pp$ strict (95\% paired Wald CI: \ci{+24.8\pp}{+45.2\pp}; paired exact McNemar $p < 10^{-8}$) and $+31.7\pp$ lenient; on a 3-level ordinal score (correct=1, partial=0.5, incorrect=0) C4g higher in 64/120, tied in 46, lower in 10 (sign test $p < 0.0001$). The keyword evaluator overestimates the strict delta by ${\sim}5$\pp\ relative to physician judgment; all three methods agree on direction.

\paragraph{Evaluator agreement.}
Keyword agrees with physician 54.2\% (Cohen's $\kappa = 0.18$)~\cite{landis1977kappa}; on the 106 items with physician-confirmed correct references, strict delta rises to $+41.4\pp$---gold-standard defects attenuate, not inflate, the benefit. C4g improves or ties on all 9 categories ($p < 0.002$); safe rate $+15.8\pp$, helpful rate $+36.7\pp$ (Appendix~\ref{app:adjudication_full}).

\subsection{External Physician Adjudication (Three-Rater)}
\label{sec:results:external}

Three physicians---Reviewer~1 (A.S., senior internal, 20+ yr), Reviewer~2 (C.H., senior external, 20+ yr), and Reviewer~3 (S.N., external resident)---independently rated the same blinded 100-item subset (50~C1, 50~C4g).

\begin{table*}[t]
\centering
\caption{Three-rater external physician adjudication ($n=100$ paired items, 50~C1 / 50~C4g, blinded). Per-reviewer C4g$-$C1 strict deltas are shown with exact-binomial McNemar $p$-values. The three-rater majority vote is reported as a sensitivity comparator; the author-blind primary endpoint is the leave-author-out paired exact McNemar (Section~\ref{sec:results:primary_endpoint}).}
\label{tab:external_adjudication}
\small
\begin{tabular}{@{}lccrl@{}}
\toprule
Reviewer (training level) & C1 strict & C4g strict & $\Delta$ & McNemar $p$ \\
\midrule
Reviewer~1 (senior internal, 20+ yr) & 28.0\% & 64.0\% & $+36.0\pp$ & 0.00028 \\
Reviewer~2 (senior external, 20+ yr) & 24.0\% & 54.0\% & $+30.0\pp$ & 0.00073 \\
Reviewer~3 (resident external) & 64.0\% & 66.0\% & $+2.0\pp$ & 1.0\phantom{000} \\
\midrule
\textbf{3-rater majority vote (2+/3)} & \textbf{40.0\%} & \textbf{64.0\%} & \textbf{$+24.0\pp$} & \textbf{0.0075} \\
\bottomrule
\end{tabular}
\end{table*}

\paragraph{Author-blind primary endpoint: leave-author-out paired exact McNemar.}
\label{sec:results:primary_endpoint}
The author-blind primary endpoint is the leave-author-out paired exact McNemar on 50 matched (C1, C4g) qid pairs from the three-rater external adjudication, restricted to Hird $\times$ Nadeem unanimous strict ratings (no inclusion of the author). C1 12/50 (24.0\%) $\to$ C4g 23/50 (46.0\%), $\Delta = +22.0\pp$ [95\% Newcombe CI: $+5.1\pp$, $+31.5\pp$], two-sided exact McNemar $p = 0.0192$ ($b=4$ favor C1, $c=15$ favor C4g). This is the substantive author-blind comparison: paired, not subject to gold-standard circularity (the test is whether independent raters agree on system output), and computed from the pre-existing v85.1 commit (\texttt{0c510d7}, 2026-04-26). Promotion of this statistic from a sensitivity to the primary endpoint occurred in v86 (2026-04-27) post-hoc, in response to external-reviewer feedback identifying single-author circularity as the dominant threat to the originally-declared $n=362$ keyword primary; both endpoints are deterministic computations over already-frozen three-rater data (commit \texttt{7b388db}, v74), and pre-registration was not performed on AsPredicted or OSF (Section~\ref{sec:bench:clinicalbench}). Inter-non-author Cohen's $\kappa = 0.36$ (Hird $\times$ Nadeem) versus $0.43$--$0.49$ for author-involving pairs.

\paragraph{Three-rater majority (sensitivity comparator).}
Under majority vote (Table~\ref{tab:external_adjudication}), C4g is correct on 32/50 vs.\ C1 on 20/50, $+24.0\pp$ ($p=0.0075$); change-excluded $+27.3\pp$. No reviewer shows an inversion on any non-\emph{change} category. Fleiss' $\kappa = 0.413$ (strict) / 0.615 (gold correctness)~\cite{landis1977kappa}; pairwise Cohen's $\kappa$ 0.36--0.49. External reviewers found 61--64\% reference defect rates (vs.\ 56\% internal). This majority result is subject to single-author circularity (Reviewer~1 is an author) and is reported as a sensitivity comparator to the primary endpoint above.

\paragraph{Structurally-independent rater.}
The only structurally-independent rater (Reviewer~3, no prior author relationship) showed $+2.0\pp$ (calibration: 65/100 vs.\ 39--46/100 for seniors); with 23 discordant pairs, minimum detectable effect is ${\sim}22\pp$, so the resident's data are consistent with effects up to that magnitude (Appendix~\ref{app:external_extended}).

\section{Discussion}
\label{sec:discussion}

\paragraph{Why long context underperforms.}
Long context forces the model to do two things at once: extract structure (negation, temporality, experiencer) and reason toward the answer. KG-RAG does the structural work upstream and hands the LLM a small set of typed facts. C6 (59.2\%) trails C4g\_oracle (68.5\%) by $-9.3\pp$ ($p = 0.001$), and the gap is concentrated in current\_state (C6 18.0\% vs.\ C4g 70.0\%). GPT-OSS 20B narrows its gap to Opus 4.6 once both get structured context (59.4\% vs.\ 66.0\%). The keyword evaluator is a reproducibility tool, not a clinical-correctness claim: it scores keyword presence and skips polarity for uncertainty, family history, conditional, current\_state, and historical questions. These shallow rules favor C4g's structured answers over C1's abstentions, so the keyword $+39.5\pp$ overestimates the true gain. The physician adjudication ($+24$ to $+36\pp$) is the comparison that matters; the keyword number is its deterministic proxy.

\paragraph{Cross-model convergence.}
All six models gain (Table~\ref{tab:crossmodel}; oracle deltas $+20.4$ to $+43.1\pp$, all $q<10^{-10}$ after BH-FDR). C4g\_oracle accuracies land in a narrow $55.8$--$66.0\%$ band even though C1 baselines span $21.8$--$39.5\%$. Regressing the C1$\to$C4g delta on the C1 baseline gives $\beta = -1.123$, $r = -0.921$, $p = 0.009$. With C4g\_oracle close to the keyword evaluator's measurable ceiling ($\sim 70\%$), Tu (\emph{BMJ} 2005) showed that paired pre-post designs can produce slopes near $-1$ as a ceiling artifact, so we cannot tell whether structured retrieval is substituting for model size or whether we are seeing regression to the mean. With $n=6$ models, this is descriptive, not inferential. One sign that representation alone is not free: C4 (assertions without routing) drops $-3.8\pp$ below C3 ($p=0.26$, n.s.; Appendix~\ref{app:c4_transitions}). Routing in C4g recovers and extends the gain.

\paragraph{Possible implications beyond clinical NLP.}
Adding metadata without aligning retrieval to it can hurt performance, as the C3$\to$C4 step shows. The C3$\to$C4$\to$C4g ablation may be useful as a test pattern wherever a system extracts structured annotations but does not integrate them into retrieval.

\paragraph{Evaluator calibration and reference-answer quality.}
The three evaluators agree directionally: keyword $+39.5\pp$, physician $+24$ to $+36\pp$, LLM-as-judge $+28.5\pp$. The keyword and LLM-judge magnitudes differ, but both compare the same C1 and C4g answers against the same gold standard, so any per-answer evaluator bias cancels in the within-model C1$\to$C4g paired delta; what survives is an architectural signal that the physician adjudication then independently corroborates. Physician review found 56\% of v2 reference answers defective. The cause is a systematic NLP assertion-classifier error: the classifier read ``history of CHF'' as ``resolved'' when clinical usage means ``chronic.'' The model is correct 59\% of the time when the reference is wrong, and 46\% of the time when it is right, so noisy labels appear to underestimate true performance rather than inflate it. We do not report a v3 gold rescore in this manuscript: the corrections were identified during physician adjudication, applying them and re-running every model$\times$condition would constitute a second freezing of the benchmark, and we prefer to keep v2 as the public test set with v3 reserved for an independent follow-up.

\paragraph{Bi-temporal versus tri-temporal modeling.}
We use the term ``tri-temporal'' loosely. The system stores bi-temporal edges (valid time and transaction time, in the Snodgrass tradition; cf.\ Graphiti~\cite{rasmussen2025zep}) and adds an NLP-derived temporality label $\tau_a \in \{\text{Past, Current, Future}\}$. The label comes from clinical NLP, not database time, so the data model is bi-temporal in the strict sense. C4 enables the assertion and temporality axes together, so we cannot separate the temporal contribution to the C4g gain. That separation is future work.

\paragraph{Statistical caveats.}
BCa cluster bootstrap with 43 patients is at the low end of reliability (Cameron \& Miller~\cite{cameron2015cluster} note $\geq 50$ clusters is preferable); CIs should be interpreted with this caveat. The hard cross-admission $n=122$ subset was selected post-hoc and merits a post-selection-inference caveat (cf.\ Berk et al.~\cite{berk2013posi}). Cross-model BH-FDR is reported; the more conservative Benjamini-Yekutieli (2001) procedure under PRDS-violating dependence yields equivalent conclusions (BY adjustment factor $2.45\times$; all $q<2.5\times 10^{-10}$).

\paragraph{Multi-site phenotype-validation roadmap.}
The current cohort comprises 43 MIMIC-IV patients (BIDMC ICU). Multi-site validation in the eMERGE / OHDSI tradition would require $\geq 5$ sites with per-site PPV. Minimum-viable transportability checks (MIMIC-III $\leftrightarrow$ MIMIC-IV intra-cohort, eICU, Synthea narrative) are flagged as future work. NLP-portability literature~\cite{emerge2023portability} projects 10--20\pp PPV degradation on transport, which we have not measured.

\paragraph{Deployment-realism scoping.}
ClinicalBench Opus C1/C3/C4g require ${\sim}22$\,min per condition per patient via API; the system has not been packaged for sub-second EHR-sidebar latency, FURM-style governance review~\cite{shah2024furm}, CHAI Assurance Reporting Checklist alignment, or post-deployment algorithmovigilance~\cite{embi2021algorithmovigilance}. These are deployment prerequisites, not paper findings; a candidate SaMD/PCCP/CHAI/algorithmovigilance scaffold is sketched in Appendix~\ref{app:regulatory}.

\paragraph{Threats to validity and claim scope.}
\begin{sloppypar}
ClinicalBench-v2 is a provisional in-distribution stress test (43 patients, 400 questions, MIMIC-IV~\cite{johnson2023mimiciv}) measuring retrieval fidelity, not external generalization or phenotype validation in the eMERGE/PheKB sense. The change-excluded keyword endpoint (now reported as a sensitivity, $n=362$) excludes 8 experiencer-attribution-defective items (Appendix~\ref{app:bugged_items}); they remain in the released v2 gold for transparency. Evaluator uncertainty is substantial (54\% keyword--physician agreement); conclusions focus on directional agreement across three evaluators. The author's combined role as designer, builder, and primary evaluator creates circularity risk structurally mitigated but not eliminated (Section~\ref{sec:results:circularity}); inter-non-author Cohen's $\kappa=0.36$ (Hird $\times$ Nadeem) further attenuates the external-adjudication signal; the leave-author-out 3-rater majority is $+22.0\pp$ (paired exact McNemar $p=0.0192$; 4 vs.\ 15 discordant pairs) versus $+24.0\pp$ with the author included, supporting both direction and significance. Model dependence is real ($+20.4$--$43.1\pp$ oracle, after BH-FDR). This work performs in-distribution evaluation and physician adjudication, not multi-site phenotype validation; it does \emph{not} establish clinical deployment readiness, and multi-site replication and a prospective clinician-with-system comparison are encouraged. Detailed threats and broader-impact discussion in Appendix~\ref{app:threats}.
\end{sloppypar}

\section{Conclusion}
\label{sec:conclusion}
The author-blind primary endpoint is leave-author-out paired exact McNemar (Hird $\times$ Nadeem unanimous strict, $n=50$; promoted post-hoc to primary in v86 from a pre-existing v85.1 sensitivity in response to external review --- not pre-registered): $+22.0\pp$ \ci{+5.1\pp}{+31.5\pp}, $p=0.0192$. The architectural novelty is C2b (Contriever dense-RAG) $\to$ C4g\_kw (intent-aware KG-RAG) on the change-excluded $n=362$ endpoint: $+8.84\pp$ (McNemar $p=1.79\times10^{-3}$). Sensitivities: three-rater physician majority $+24.0\pp$ ($p=0.0075$; subject to single-author circularity since the author is R1); deterministic keyword proxy $+39.5\pp$ (reproducibility proxy only, not a clinical correctness claim).
Cross-model accuracies converge with strong negative slope vs C1 baseline ($\beta=-1.123$, $r=-0.921$, $p=0.009$), consistent with regression to the mean rather than encoding substituting for parameter count.
The $56\%$ reference-answer defect rate underscores a methodological lesson: automated NLP-pipeline benchmarks require physician adjudication.
Reasoning benchmarks like HealthBench measure the last mile given clean vignettes; retrieval-faithfulness on real charts is the first mile, and representation is the undermeasured prerequisite.

\FloatBarrier

\section*{Acknowledgments}
The author thanks non-author contributors Cindy Hird, MD and Shaheera Nadeem, MD for independent physician adjudication of ClinicalBench items, and Yu Tian, PhD (University of Central Florida) for feedback on an early draft.
MIMIC-IV data were accessed under PhysioNet Credentialed Health Data Use Agreement.

\paragraph{Funding.} This research received no specific grant from any funding agency in the public, commercial, or not-for-profit sectors.

\paragraph{Conflict of Interest.} The author serves as founder of Sulci.ai, a clinical-AI startup that is exploring deployment of EpiKG-derived technology. No commercial relationship funded this manuscript or the underlying experiments. The University of Central Florida is the institution of academic record for this work.

\paragraph{Author Contributions.} Alex Stinard, MD: Conceptualization, Methodology, Software, Validation, Formal Analysis, Investigation, Data Curation, Writing -- Original Draft, Writing -- Review \& Editing, Visualization, Project Administration.

\paragraph{Data Availability.} ClinicalBench questions, reference answers (v1 and v2), raw model predictions, the deterministic keyword evaluator, and physician adjudication data are publicly available at \url{https://huggingface.co/datasets/alexstinard/epikg-clinicalbench} (DOI: \href{https://doi.org/10.57967/hf/8549}{10.57967/hf/8549})~\cite{stinard2026clinicalbench}.
MIMIC-IV clinical notes require separate credentialed access via PhysioNet (\url{https://physionet.org/content/mimiciv/3.1/}).

\paragraph{Use of AI Tools.} Claude Code (Anthropic, Claude Opus 4.6) was used as a programming assistant during system development, data analysis, and manuscript preparation.
All AI-generated content was reviewed, verified, and edited by the author.
Claude Opus 4.6 is also the primary answering model evaluated in ClinicalBench experiments; this dual role is disclosed throughout the paper.

\bibliography{references}

\onecolumn
\appendix

\let\oldsection\section
\renewcommand{\section}{\FloatBarrier\oldsection}

\section{Notation Summary}
\label{app:notation}

\begin{table}[ht]
\centering
\small
\begin{tabular}{@{}ll@{}}
\toprule
\textbf{Symbol} & \textbf{Meaning} \\
\midrule
$\alpha \in \mathcal{A}$ & Assertion label (7 values, Eq.~\ref{eq:assertion}) \\
$\boldsymbol{\tau}_v$ & Valid time (event date, valid from/to) \\
$\boldsymbol{\tau}_t$ & Transaction time (recorded at, doc date, created at) \\
$\tau_a$ & NLP-asserted temporality (Current / Past / Future) \\
$r \in \mathcal{R}$ & Temporal relation (9 values, Table~\ref{tab:temporal_relations}) \\
$c$ & Confidence score $\in [0,1]$ \\
$\xi$ & Experiencer (Patient / Family) \\
$e(m)$ & Epistemic state tuple $(c, \alpha, \xi, \tau)$ \\
$q, \pi$ & Clinical question, patient \\
$\iota$ & Question intent (Change / Current\_State / Historical / Default) \\
$f_{\textit{np}}(c)$ & Fraction of non-present assertions for concept $c$ \\
\bottomrule
\end{tabular}
\caption{Key notation used throughout the paper.}
\label{tab:notation}
\end{table}

\section{Formal Epistemic Preservation}
\label{app:formal}

The epistemic invariant is formalized for the epistemic invariant maintained by the \epikg pipeline, building on the principle that aboutness---the relationship between information artifacts and the entities they represent---must be preserved across transformations~\cite{ceusters2015ontologyepistemology}.

\begin{definition}[Epistemic State]
\label{def:epistemic}
For a clinical mention $m$, the \emph{epistemic state} is the tuple $e(m) = (c, \alpha, \xi, \tau)$, where $c$ is the OMOP concept identifier, $\alpha \in \mathcal{A}$ is the 7-value assertion label (Eq.~\ref{eq:assertion}), $\xi \in \{\textnormal{\textsc{Patient}}, \textnormal{\textsc{Family}}\}$ is the experiencer, and $\tau \in \{\textnormal{\textsc{Current}}, \textnormal{\textsc{Past}}, \textnormal{\textsc{Future}}\}$ is the temporality.
A pipeline $P$ \emph{epistemically preserves} mention $m$ if $P(e(m)) = e(m)$; that is, the epistemic state at the output of every pipeline stage is identical to the state at extraction.
\end{definition}

\begin{proposition}[Assertion Entropy Loss]
\label{prop:entropy}
For concept $c$ in patient $\pi$'s record, let $A_c^\pi$ denote the assertion distribution with empirical frequencies $p(\alpha_i) = n_i / \sum_j n_j$ over the $|\mathcal{A}|$ assertion classes.
The assertion entropy is
\begin{equation}
\label{eq:entropy}
H(A_c^\pi) = -\sum_{i} p(\alpha_i) \log p(\alpha_i) \;.
\end{equation}
An assertion-blind pipeline collapses all mentions to $\alpha = \textsc{Present}$, yielding a degenerate distribution with $H = 0$.
The information loss is $\Delta H = H(A_c^\pi) \geq 0$~\cite{shannon1948}, with $\Delta H > 0$ strictly whenever any mention carries a non-present assertion.
\end{proposition}

\begin{proof}
Under the assertion-blind mapping $\phi: \alpha_i \mapsto \textsc{Present}$ for all $i$, the output distribution assigns probability 1 to \textsc{Present} and 0 to all other classes, so $H(\phi(A_c^\pi)) = 0$.
By non-negativity of Shannon entropy, $\Delta H = H(A_c^\pi) - 0 = H(A_c^\pi) \geq 0$, with equality iff all mentions already carry $\alpha = \textsc{Present}$.
\end{proof}

\begin{corollary}[Faithfulness Bound]
\label{cor:faithfulness}
Let $f_{\textit{np}}(c)$ denote the fraction of mentions of concept $c$ carrying a non-present assertion ($\alpha \neq \textsc{Present}$).
Without assertion labels, the maximum assertion-faithful accuracy for any downstream task conditioned on $c$ is bounded by $1 - f_{\textit{np}}(c)$.
In clinical records where negated and uncertain mentions are prevalent---\eg ``no pneumonia,'' ``possible CHF''---this bound can be substantially below 1.
\end{corollary}

\begin{proof}
Without assertion labels, any predictor must assign a single assertion class to all mentions of $c$.
Choosing $\textsc{Present}$ (the majority class in clinical text) yields accuracy $1 - f_{\textit{np}}(c)$; the $f_{\textit{np}}(c)$ non-present mentions are necessarily misclassified.
\end{proof}

\section{Extended System Design}
\label{app:system_extended}

\subsection{Temporal Knowledge Graph Model (Bi-Temporal Storage + NLP-Asserted Label)}
\label{app:temporal_model}

Clinical events unfold across multiple time dimensions that prior systems model incompletely~\cite{huang2024medtkg,rasmussen2025zep}.
Bi-temporal databases~\cite{snodgrass2000temporal} distinguish valid time from transaction time; \epikg's per-edge representation adds an NLP-asserted temporality label $\tau_a$ and stores all three on every KG edge (data-modeling clarification below):

\begin{definition}[Temporal Edge --- bi-temporal storage with NLP-asserted label]
\label{def:edge}
An edge $e = (s, p, o, \alpha, \boldsymbol{\tau}_v, \boldsymbol{\tau}_t, \tau_a, r, c)$ where:
\begin{itemize}[nosep,leftmargin=1.5em]
  \item $s, o$ are source and target nodes; $p$ is the predicate (one of 24 edge types);
  \item $\alpha \in \mathcal{A}$ is the assertion label (Eq.~\ref{eq:assertion});
  \item $\boldsymbol{\tau}_v = (\textit{event\_date}, \textit{valid\_from}, \textit{valid\_to})$ is the \textbf{valid time} interval---when the relationship held in the real world;
  \item $\boldsymbol{\tau}_t = (\textit{recorded\_at}, \textit{doc\_date}, \textit{created\_at})$ is the \textbf{transaction time}---when it was recorded;
  \item $\tau_a \in \{\textnormal{\textsc{Current}}, \textnormal{\textsc{Past}}, \textnormal{\textsc{Future}}\}$ is the \textbf{NLP-asserted temporality};
  \item $r \in \mathcal{R}$ is an Allen interval algebra relation~\cite{allen1983temporal}; $c \in [0,1]$ is temporal confidence.
\end{itemize}
\end{definition}

$\mathcal{R}$ comprises seven of Allen's 13 canonical relations~\cite{allen1983temporal} (merging symmetric pairs like Before/Meets) plus \textsc{Concurrent} and \textsc{Unknown} (nine total; full mapping in Table~\ref{tab:temporal_relations}).
Unlike TEO~\cite{li2020teo} and TCL~\cite{ong2023tcl}, which use Allen's relations as annotation-ontology classes or modal operators, \epikg stores $r$ and $c$ directly on materialized edges; this is intended to enable temporal-interval filtering during traversal, although the current intent-aware retrieval algorithm queries the categorical $\tau_a$ label rather than Allen-relation predicates (see clarification below).

\paragraph{Data-modeling clarification (bi-temporal storage + derived label).}
What we describe as ``tri-temporal'' is more precisely a \emph{bi-temporal storage layer} (valid time $\boldsymbol{\tau}_v$ + transaction time $\boldsymbol{\tau}_t$, in the Snodgrass tradition~\cite{snodgrass2000temporal}; cf.\ Graphiti~\cite{rasmussen2025zep}) plus an \emph{NLP-asserted temporality label} $\tau_a \in \{\textsc{Past}, \textsc{Current}, \textsc{Future}\}$.
The label is derived from clinical NLP rather than from database time, so on data-modeling grounds the model is bi-temporal-with-derived-attribute, not strictly tri-temporal.
Allen-style interval relations $r \in \mathcal{R}$ are stored on edges as metadata, but the current intent-aware retrieval algorithm queries the categorical $\tau_a$ label rather than Allen-relation predicates; using stored intervals in retrieval (e.g., to enforce \textsc{Before}/\textsc{Overlaps} constraints) is flagged as future work.

\subsection{Graph-Augmented Retrieval Details}
\label{app:retrieval_details}

Given a clinical question $q$ and patient $\pi$, the base retrieval pipeline:
(1)~extracts concepts from $q$ via NLP + OMOP enrichment;
(2)~traverses 2--3 hops via bidirectional breadth-first search (BFS) over patient KG edges and OMOP vocabulary relationships (20M+ edges) via PostgreSQL common table expressions (CTEs) (Appendix~\ref{app:drknows}), pruning edges with $c < 0.3$;
(3)~groups edges into four temporal views (event timeline, current state, historical, conflicts);
(4)~retrieves matching guideline sections;
(5)~scores edges by:
\begin{equation}
\label{eq:scoring}
\text{score}(e) = c(e) + \mathbb{1}[\text{type}(e) \in Q_{\text{rel}}] \cdot 0.2 + \mathbb{1}[\tau_a(e) = \textsc{Current}] \cdot 0.1
\end{equation}
where the bonuses for question-relevant edge types ($0.2$) and current temporality ($0.1$) were set by manual tuning on a development set of 20 questions.
The surviving subgraph is serialized into structured text preserving assertion labels (\eg ``\textsc{Absent}: pneumonia'');
and (6)~composes graph evidence, temporal context, guidelines, and source documents into a single prompt.

\section{Gap Analysis}
\label{app:gap}

\begin{table}[ht]
\centering
\caption{Capability comparison across representative systems. \smark{} = supported, $\circ$ = partial, \xmark{} = not supported.}
\label{tab:gap}
\small
\setlength{\tabcolsep}{3.0pt}
\begin{tabular}{@{}lcccccccc@{}}
\toprule
\textbf{System} & \rotatebox{70}{\textbf{Patient KG}} & \rotatebox{70}{\textbf{OMOP mapping}} & \rotatebox{70}{\textbf{Assertion (7-class)}} & \rotatebox{70}{\textbf{Bi-temp.+label}} & \rotatebox{70}{\textbf{Assert.-aware RAG}} & \rotatebox{70}{\textbf{Experiencer}} & \rotatebox{70}{\textbf{Allen's algebra}} & \rotatebox{70}{\textbf{Multi-hop ($\geq$3)}} \\
\midrule
DoctorRAG~\cite{lu2025doctorrag}         & \xmark & \xmark & \xmark & \xmark & \xmark & \xmark & \xmark & \xmark \\
GFM-RAG~\cite{luo2025gfmrag}             & $\circ$ & \xmark & \xmark & \xmark & \xmark & \xmark & \xmark & \smark{} \\
MedRAG~\cite{wang2025medrag}              & $\circ$ & \xmark & \xmark & \xmark & \xmark & \xmark & \xmark & $\circ$ \\
KARE~\cite{jiang2025kare}                 & $\circ$ & \xmark & \xmark & \xmark & $\circ$ & \xmark & \xmark & \smark{} \\
Multi-LLM KG~\cite{chen2026multilllmkgrag} & \xmark & \smark{} & $\circ$ & \xmark & \xmark & \xmark & \xmark & \xmark \\
MedTKG~\cite{huang2024medtkg}             & \smark{} & \xmark & \xmark & $\circ$ & \xmark & \xmark & \xmark & \smark{} \\
Graphiti~\cite{rasmussen2025zep}       & \xmark & \xmark & \xmark & $\circ$ & \xmark & \xmark & \xmark & \smark{} \\
MediGRAF~\cite{thio2025medigraf}          & \smark{} & \xmark & \xmark & \xmark & \xmark & \xmark & \xmark & \smark{} \\
\midrule
\textbf{\epikg (this work)}                    & \smark{} & \smark{} & \smark{} & \smark{} & \smark{} & \smark{} & \smark{} & $\circ$ \\
\bottomrule
\end{tabular}
\end{table}

Partial marks indicate limited coverage: GFM-RAG, MedRAG, and KARE build population-level or hierarchical KGs (not per-patient cross-admission graphs); KARE uses KG-augmented retrieval but without assertion awareness; Multi-LLM KG models uncertainty via entropy but not the full assertion taxonomy; MedTKG and Graphiti carry bi-temporal annotations but lack the assertion-aware retrieval of \epikg.
MediGRAF is the closest patient-level competitor, constructing per-patient medical graphs from clinical notes, but it does not preserve assertion status or temporal relations as edge attributes.
\epikg's partial mark on multi-hop traversal ($\geq 3$ hops) reflects a deliberate architectural trade-off: PostgreSQL CTE-based traversal provides ACID compliance but degrades beyond 2~hops (Appendix~\ref{app:drknows}).

\subsection{Extended Related Work Discussion}
\label{app:related_extended}

This subsection expands the four-axis related work overview in Section~\ref{sec:related}.

\paragraph{Medical RAG systems.}
Graph-augmented retrieval has emerged as a leading paradigm for medical QA.
GraphRAG~\cite{edge2024graphrag} introduced community-based summarization; GFM-RAG~\cite{luo2025gfmrag} trained a graph foundation model across 60 KGs; KARE~\cite{jiang2025kare} adapted community retrieval to clinical decision support; Medical-Graph-RAG~\cite{wu2025medicalgraphrag} links documents via triple graphs.
However, no prior clinical KG-RAG system jointly propagates note-derived assertion classes together with bi-temporal storage plus an NLP-asserted temporality label as first-class graph properties through extraction, storage, and retrieval.
GFM-RAG and KARE build population-level graphs rather than patient-level graphs from clinical text, and a recent survey~\cite{graphragsurvey2025} does not identify systems that carry epistemic metadata through the full retrieval stack.

\paragraph{Clinical QA benchmarks and systems.}
MedPaLM~2~\cite{singhal2023medpalm2}, Med-Gemini~\cite{saab2024medgemini}, and AMIE~\cite{tu2024amie} target \emph{medical knowledge recall} from parametric knowledge or literature.
ClinicalBench targets a narrower task: \emph{assertion-faithful cross-admission reasoning} over real EHR records, where the challenge is knowing whether \emph{this patient} is still on metformin, whether a condition was ruled out or confirmed, and how the clinical picture changed across admissions.
Prior EHR QA benchmarks---emrQA~\cite{pampari2018emrqa}, EHR-RAG~\cite{cao2026ehrrag}, MIRAGE~\cite{xiong2024mirage}---evaluate medically grounded retrieval but not assertion-sensitive, cross-admission QA with category$\times$condition ablations and physician adjudication.

\paragraph{Clinical KG construction.}
Multi-LLM KG-RAG~\cite{chen2026multilllmkgrag} uses multi-agent prompting with schema-constrained extraction for oncology; AutoRD~\cite{li2024autord} and RECAP-KG~\cite{remy2024recapkg} apply LLMs to rare disease and GP notes respectively.
All share a common limitation: assertion status is not propagated into the final graph.

\paragraph{Assertion detection.}
NegEx~\cite{chapman2001negex} introduced trigger-based negation; ConText~\cite{harkema2009context} extended it with temporality and experiencer; Gul et al.~\cite{gul2025beyondnegation} fine-tuned LLMs to 0.962 accuracy on the i2b2/VA taxonomy~\cite{uzuner2011i2b2assertion}.
All treat assertion detection as a terminal annotation task: labels are not carried into knowledge graphs or retrieval systems, and are not temporally situated across encounters.

\paragraph{Temporal knowledge graphs.}
MedTKG~\cite{huang2024medtkg} constructs temporal KGs with time-stamped snapshots (event time only); Graphiti~\cite{rasmussen2025zep} implements bitemporal edges but lacks clinical ontology alignment.

\section{Safety Score}
\label{app:safety}

The safety score $S = 1 - \frac{1}{N} \sum_{i=1}^{N} w_i \cdot \mathbb{1}[\hat{y}_i \neq y_i]$ applies $w_i = 2.0$ for false-positive assertion errors (reporting a negated or absent condition as present) and $w_i = 1.0$ otherwise, capturing asymmetric clinical risk where acting on a falsely affirmed condition is more dangerous than missing a present one.
The value $w=2.0$ is treated as a reasonable default; sensitivity to this choice is a direction for future work.

\section{Cross-Model Per-Category Results}
\label{app:crossmodel_categories}

\begin{table}[ht]
\centering
\caption{ClinicalBench per-category accuracy (\%) by model under C1 and C4g (intent-aware KG-RAG, evaluator v2) for Opus, MedGemma 27B, GPT-OSS, Gemma~4 31B, MedGemma 1.5 4B, and Qwen3.5. These per-category improvements are descriptive; several categories are small, and the change category has known label/evaluator defects discussed in Section~\ref{sec:results:physician}.}
\label{tab:crossmodel_categories}
\scriptsize
\setlength{\tabcolsep}{1.5pt}
\begin{tabular}{@{}lcccccccccccccccccccc@{}}
\toprule
& & \multicolumn{3}{c}{Claude Opus 4.6} & \multicolumn{3}{c}{MedGemma 27B} & \multicolumn{3}{c}{GPT-OSS 20B} & \multicolumn{3}{c}{Gemma~4 31B} & \multicolumn{3}{c}{MedGemma 1.5 4B} & \multicolumn{3}{c}{Qwen3.5 35B} \\
\cmidrule(lr){3-5} \cmidrule(lr){6-8} \cmidrule(lr){9-11} \cmidrule(lr){12-14} \cmidrule(lr){15-17} \cmidrule(lr){18-20}
Category & $n$ & C1 & C4g & $\Delta$ & C1 & C4g & $\Delta$ & C1 & C4g & $\Delta$ & C1 & C4g & $\Delta$ & C1 & C4g & $\Delta$ & C1 & C4g & $\Delta$ \\
\midrule
Negation & 110 & 44.5 & 81.8 & \up{37.3} & 57.3 & 77.3 & \up{20.0} & 45.5 & 80.9 & \up{35.5} & 74.5 & 82.7 & \up{8.2} & 59.1 & 64.5 & \up{5.4} & 69.1 & 85.5 & \up{16.4} \\
Conditional & 20 & 0.0 & 45.0 & \up{45.0} & 15.0 & 35.0 & \up{20.0} & 0.0 & 15.0 & \up{15.0} & 0.0 & 45.0 & \up{45.0} & 5.0 & 25.0 & \up{20.0} & 5.0 & 20.0 & \up{15.0} \\
Uncertainty & 40 & 12.5 & 50.0 & \up{37.5} & 7.5 & 32.5 & \up{25.0} & 5.0 & 30.0 & \up{25.0} & 17.5 & 30.0 & \up{12.5} & 17.5 & 45.0 & \up{27.5} & 37.5 & 32.5 & \dn{5.0} \\
Family hist. & 30 & 3.3 & 56.7 & \up{53.3} & 3.3 & 40.0 & \up{36.7} & 3.3 & 40.0 & \up{36.7} & 23.3 & 83.3 & \up{60.0} & 3.3 & 33.3 & \up{30.0} & 26.7 & 40.0 & \up{13.3} \\
Sequence & 40 & 7.5 & 62.5 & \up{55.0} & 2.5 & 37.5 & \up{35.0} & 7.5 & 32.5 & \up{25.0} & 17.5 & 40.0 & \up{22.5} & 40.0 & 75.0 & \up{35.0} & 25.0 & 57.5 & \up{32.5} \\
Current st. & 50 & 26.0 & 70.0 & \up{44.0} & 32.0 & 58.0 & \up{26.0} & 36.0 & 56.0 & \up{20.0} & 12.0 & 66.0 & \up{54.0} & 36.0 & 58.0 & \up{22.0} & 36.0 & 56.0 & \up{20.0} \\
Duration & 30 & 33.3 & 60.0 & \up{26.7} & 46.7 & 63.3 & \up{16.7} & 16.7 & 66.7 & \up{50.0} & 60.0 & 63.3 & \up{3.3} & 53.3 & 86.7 & \up{33.4} & 40.0 & 86.7 & \up{46.7} \\
Historical & 50 & 6.0 & 60.0 & \up{54.0} & 0.0 & 46.0 & \up{46.0} & 0.0 & 60.0 & \up{60.0} & 12.0 & 36.0 & \up{24.0} & 12.0 & 36.0 & \up{24.0} & 8.0 & 46.0 & \up{38.0} \\
Change & 30 & 10.0 & 96.7 & \up{86.7} & 0.0 & 73.3 & \up{73.3} & 0.0 & 66.7 & \up{66.7} & 3.3 & 36.7 & \up{33.4} & 6.7 & 26.7 & \up{20.0} & 3.3 & 16.7 & \up{13.3} \\
\midrule
\textbf{Overall} & \textbf{400} & \textbf{21.8} & \textbf{68.5} & \up{\textbf{46.8}} & \textbf{25.5} & \textbf{56.5} & \up{\textbf{31.0}} & \textbf{20.2} & \textbf{60.2} & \up{\textbf{40.0}} & \textbf{33.5} & \textbf{58.5} & \up{\textbf{25.0}} & \textbf{33.0} & \textbf{53.8} & \up{\textbf{20.8}} & \textbf{36.2} & \textbf{57.0} & \up{\textbf{20.7}} \\
\bottomrule
\end{tabular}
\end{table}

\section{MedGemma Full Per-Category Results}
\label{app:medgemma_percategory}

\begin{table}[ht]
\centering
\small
\begin{tabular}{@{}lcccc@{}}
\toprule
\textbf{Category} & $n$ & \textbf{C7} & \textbf{C1} & \textbf{C4g} \\
\midrule
Negation & 110 & --- & 57.3 & 77.3 \\
Conditional & 20 & --- & 15.0 & 35.0 \\
Uncertainty & 40 & --- & 7.5 & 32.5 \\
Family History & 30 & --- & 3.3 & 40.0 \\
Sequence & 40 & --- & 2.5 & 37.5 \\
Current State & 50 & --- & 32.0 & 58.0 \\
Duration & 30 & --- & 46.7 & 63.3 \\
Historical & 50 & --- & 0.0 & 46.0 \\
Change & 30 & --- & 0.0 & 73.3 \\
\midrule
\textbf{Overall} & \textbf{400} & --- & \textbf{25.2} & \textbf{56.2} \\
\bottomrule
\end{tabular}
\caption{MedGemma 27B ClinicalBench per-category accuracy (\%, keyword evaluator v2) for three conditions. Per-category changes are descriptive; several categories are small, and the change category has known label/evaluator defects discussed in Section~\ref{sec:results:physician}.}
\label{tab:medgemma_percategory}
\end{table}

\subsection{MedGemma 1.5 4B Full Per-Category Results}
\label{app:medgemma15_percategory}

\begin{table}[ht]
\centering
\small
\begin{tabular}{@{}lcccc@{}}
\toprule
\textbf{Category} & $n$ & \textbf{C1} & \textbf{C4g KW} & \textbf{C4g Oracle} \\
\midrule
Negation & 110 & 59.1 & 64.5 & 64.5 \\
Conditional & 20 & 5.0 & 25.0 & 25.0 \\
Uncertainty & 40 & 17.5 & 45.0 & 45.0 \\
Family History & 30 & 3.3 & 23.3 & 33.3 \\
Sequence & 40 & 40.0 & 75.0 & 75.0 \\
Current State & 50 & 36.0 & 40.0 & 58.0 \\
Duration & 30 & 53.3 & 86.7 & 86.7 \\
Historical & 50 & 12.0 & 36.0 & 36.0 \\
Change & 30 & 6.7 & 26.7 & 26.7 \\
\midrule
\textbf{Overall} & \textbf{400} & \textbf{33.0} & \textbf{50.7} & \textbf{53.8} \\
\bottomrule
\end{tabular}
\caption{MedGemma 1.5 4B ClinicalBench per-category accuracy (\%, keyword evaluator v2) for three conditions: C1 (LLM alone), C4g keyword-only classification, and C4g oracle classification. Keyword$\to$oracle delta: $+3.0\pp$ overall, concentrated in current state ($+18\pp$) and family history ($+10\pp$); all other categories are identical between oracle and keyword. Per-category changes are descriptive; several categories are small, and the change category has known label/evaluator defects discussed in Section~\ref{sec:results:physician}.}
\label{tab:medgemma15_percategory}
\end{table}

\section{Experiencer Attribute Propagation Ablation}
\label{app:experiencer}

The \texttt{experiencer} attribute distinguishes patient conditions from family member conditions; without it, the graph conflates the two, causing family history misattribution.

\begin{table}[ht]
\centering
\caption{Impact of experiencer attribute propagation on Opus C4 (400 questions, prior evaluator run). Categories with no change omitted.}
\label{tab:experiencer_fix}
\small
\begin{tabular}{@{}lccc@{}}
\toprule
Category & Pre-fix & Post-fix & $\Delta$ \\
\midrule
Family history & 63.3\% & \textbf{73.3\%} & $+10.0\pp$ \\
Uncertainty & 50.0\% & \textbf{55.0\%} & $+5.0\pp$ \\
Historical & 34.0\% & 38.0\% & $+4.0\pp$ \\
Current state & 46.0\% & 48.0\% & $+2.0\pp$ \\
\midrule
\textbf{Overall C4} & \textbf{68.2\%} & \textbf{70.2\%} & $+2.0\pp$ \\
\bottomrule
\end{tabular}
\end{table}

The fix improved family history by $+10.0\pp$ with zero regressions on guard categories (negation, conditional, duration, sequence all unchanged), confirming that the experiencer attribute is load-bearing for assertion-sensitive reasoning.

\section{SliceBench}
\label{app:slicebench}
\label{app:slicebench_figures}

SliceBench (6 patients, 144 questions, 3 complexity tiers; LLM-as-judge evaluation) is consistent with a complexity-dependent KG effect: the incremental KG layer (B2$\rightarrow$B3) contributes $+2.2\pp$ overall (CI: \ci{-1.5\pp}{+5.9\pp}), not reaching aggregate significance, while Tier~C (15+ encounters) gains $+5.0\pp$ vs.\ Tier~A (1--2 notes) $+0.6\pp$ (Table~\ref{tab:slicebench_tiers}).
Because the overall B2$\rightarrow$B3 comparison is not statistically distinguishable from zero and tier-level results are descriptive only ($n=2$ patients per tier), this pattern is treated as exploratory.

\begin{figure}[!htbp]
\centering
\includegraphics[width=0.75\textwidth]{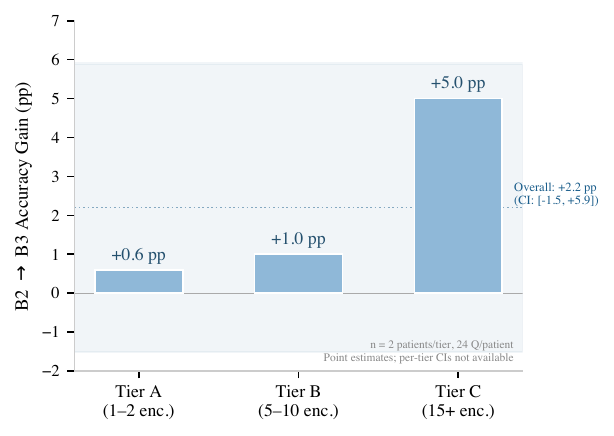}
\vspace{4pt}
\includegraphics[width=\textwidth]{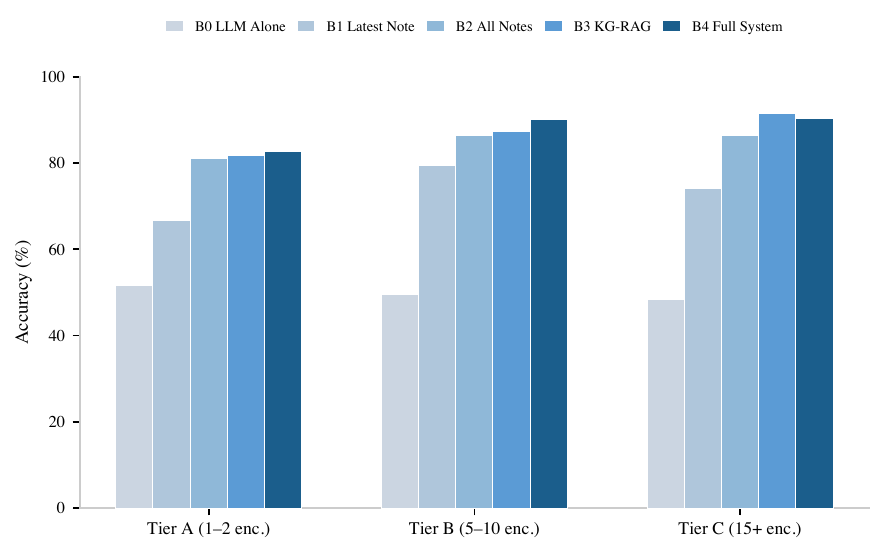}
\caption{SliceBench exploratory results ($n=2$ patients per tier, 24 questions each). Tier-specific B2$\rightarrow$B3 deltas are point estimates only and should not be over-interpreted; the overall B2$\rightarrow$B3 delta is $+2.2\pp$ (95\% CI \ci{-1.5\pp}{+5.9\pp}).}
\label{fig:slicebench_figures}
\end{figure}

\section{Per-Category Delta Chart and Transition Analysis}
\label{app:deltas}

\begin{figure}[!htbp]
\centering
\includegraphics[width=\textwidth]{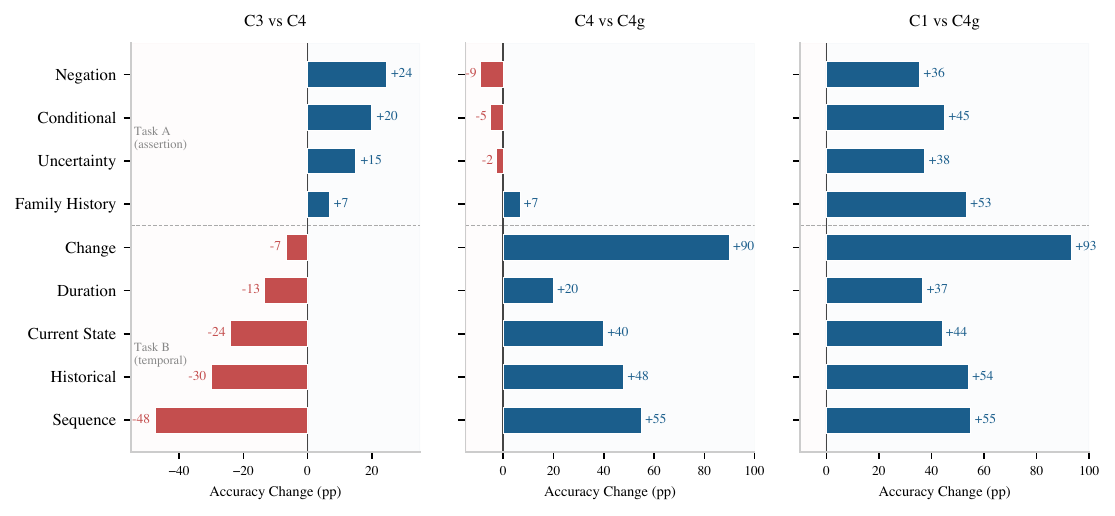}
\caption{Per-category deltas for three paired contrasts (C3 vs.\ C4, C4 vs.\ C4g, and C1 vs.\ C4g). These contrasts are descriptive, not a formal additive decomposition, and categories with small $n$ or known evaluator/reference-answer defects are shown for transparency only.}
\label{fig:deltas}
\end{figure}

\label{app:c4_transitions}

Table~\ref{tab:c4_transitions} reports question-level transitions between C3 (KG-RAG without assertions) and C4 (with assertions, no routing), with C4g recovery rates.

\begin{table}[ht]
\centering
\small
\caption{Per-category C3$\to$C4 transition analysis (Claude Opus 4.6, $n\!=\!400$). \emph{Regr.}: C3 correct $\to$ C4 incorrect; \emph{Impr.}: C3 incorrect $\to$ C4 correct; \emph{Recov.}: regressions recovered by C4g.}
\label{tab:c4_transitions}
\begin{tabular}{@{}lrrrrrr@{}}
\toprule
Category & $n$ & C3 & C4 & Regr. & Impr. & Recov. \\
\midrule
\multicolumn{7}{@{}l}{\emph{Temporal categories (net $-52$):}} \\
Current state & 50 & 54\% & 30\% & 20 & 8 & 17 \\
Sequence & 40 & 55\% & 10\% & 19 & 1 & 19 \\
Historical & 50 & 42\% & 12\% & 18 & 3 & 15 \\
Duration & 30 & 57\% & 47\% & 10 & 7 & 8 \\
Change & 30 & 20\% & 7\% & 6 & 2 & 5 \\
\midrule
\multicolumn{7}{@{}l}{\emph{Assertion-sensitive categories (net $+37$):}} \\
Negation & 110 & 66\% & 89\% & 6 & 31 & 5 \\
Family history & 30 & 43\% & 50\% & 8 & 10 & 7 \\
Uncertainty & 40 & 38\% & 52\% & 7 & 13 & 6 \\
Conditional & 20 & 30\% & 50\% & 1 & 5 & 1 \\
\midrule
\textbf{Total} & 400 & 50\% & 46\% & 95 & 80 & 83 \\
\bottomrule
\end{tabular}
\end{table}

\noindent Overall, 87.4\% of regressions (83/95) are recovered by C4g. Temporal categories account for 73/95 regressions but only 21/80 improvements; assertion-sensitive categories show the reverse pattern (22/95 regressions, 59/80 improvements). This confirms that assertions help on epistemic questions but suppress needed evidence for temporal synthesis without intent-matched routing.

\paragraph{Qualitative examples.}
C4 regressions exhibit a consistent pattern: the evidence context includes raw assertion summaries that distract the LLM from the question's intent.
For family-history questions, C4 states ``The patient has Lymphoma'' (leaking family-history findings to patient state); for sequence questions, C4 lists assertion types instead of temporal ordering; for historical questions, C4 notes ``history of depression'' and concludes it is no longer active.
In each case, C4g's intent-aware retrieval eliminates the distraction by routing to category-specific evidence (cross-admission comparison, timeline traversal, or current-state filtering).

\section{SliceBench Conditions}
\label{app:slicebench_conditions}

SliceBench selects patients in three tiers: \textbf{Tier~A} (2 patients, 1--2 encounters), \textbf{Tier~B} (2 patients, 5--10), and \textbf{Tier~C} (2 patients, 15+).
Each patient receives 24 questions spanning hard cross-admission categories (cross-encounter medication timelines, problem list reconciliation, causal chain tracing), yielding 144 total.

\begin{table}[ht]
\centering
\small
\caption{SliceBench conditions. B0--B4 form a monotone context progression.}
\label{tab:slicebench_conditions}
\begin{tabular}{@{}clp{7.5cm}@{}}
\toprule
ID & Condition & Context Provided to LLM \\
\midrule
B0 & LLM Alone & Question only (parametric knowledge) \\
B1 & Latest Note & Most recent clinical note \\
B2 & All Notes RAG & All notes, retrieved by relevance \\
B3 & KG-RAG & Knowledge graph context + all notes \\
B4 & Full System & KG-RAG + guidelines + calculators \\
\bottomrule
\end{tabular}
\end{table}

\section{System Implementation Details}
\label{app:supporting}

\paragraph{LLM baseline.}
On MedQA-USMLE (965 questions), Claude Opus 4.5 achieves 81.6\% accuracy (5.1\,pp below GPT-4 at 86.7\% and Med-PaLM~2 at 86.5\%~\cite{medagentbench2025}), establishing that the Claude model family---Sonnet 4.5 (SliceBench answerer) and Opus 4.6 (ClinicalBench primary, SliceBench judge)---provides a reasonable LLM baseline (Table~\ref{tab:medqa}).

\begin{table}[ht]
\centering
\small
\begin{tabular}{lccc}
\toprule
\textbf{Model} & \textbf{Accuracy} & \textbf{Step 1} & \textbf{N} \\
\midrule
EpiKG (LLM alone) & 81.6\% & 79.9\% & 965 \\
GPT-4 (2023) & 86.7\% & --- & 1{,}273 \\
Med-PaLM 2 (2023) & 86.5\% & --- & 1{,}273 \\
Claude 3 Opus (2024) & 78.0\% & --- & 1{,}273 \\
\bottomrule
\end{tabular}
\caption{MedQA-USMLE results. The LLM-alone baseline is competitive.}
\label{tab:medqa}
\end{table}

\paragraph{KG scale and traversal latency.}
The deployed system processes 145 documents across 85 patients, materializing 3{,}100 KG nodes and 8{,}803 edges.
Graph traversal latency is sub-millisecond at this scale: 0.57\,ms (1-hop), 0.75\,ms (2-hop).
The full system integrates 201 clinical calculators, 1{,}202 guideline sections, 20M+ OMOP vocabulary relationships, 122 assertion trigger patterns, and 24 edge types across 13 node types.

\paragraph{Multi-hop traversal.}
\label{app:drknows}
On the DR.KNOWS diagnostic reasoning benchmark~\cite{drknows2025}, which measures KG traversal accuracy using PostgreSQL-backed clinical data, \epikg achieves 0.420 overall (50\% at 1-hop, 25\% at 2-hop, 0\% at 3-hop).
Multi-hop degradation reflects a deliberate architectural trade-off: PostgreSQL CTE-based traversal provides ACID compliance but scales poorly beyond 2~hops.

\section{SliceBench Tier-Stratified Results}
\label{app:slicebench_tiers}

\begin{table}[ht]
\centering
\small
\caption{SliceBench results stratified by patient complexity tier. B2$\rightarrow$B3 $\Delta$ shows the incremental KG contribution.}
\label{tab:slicebench_tiers}
\begin{tabular}{@{}lccccccc@{}}
\toprule
& & \multicolumn{5}{c}{Score (\%)} & \\
\cmidrule(lr){3-7}
Tier & Encounters & B0 & B1 & B2 & B3 & B4 & B2$\rightarrow$B3 $\Delta$ \\
\midrule
A & 1--2 notes & 51.7 & 66.7 & 81.1 & 81.8 & 82.6 & $+0.6\pp$ \\
B & 5--10 notes & 49.5 & 79.4 & 86.4 & 87.4 & 90.1 & $+1.0\pp$ \\
C & 15+ notes & 48.4 & 74.2 & 86.5 & 91.5 & 90.3 & $+5.0\pp$ \\
\bottomrule
\end{tabular}
\end{table}

\section{ClinicalBench Full Per-Category Results (Opus)}
\label{app:percategory}

\begin{table}[ht]
\centering
\small
\begin{tabular}{lccccccccc}
\toprule
\textbf{Category} & $n$ & \textbf{C7} & \textbf{C6} & \textbf{C1} & \textbf{C2} & \textbf{C2b} & \textbf{C3} & \textbf{C4} & \textbf{C4g} \\
\midrule
Negation & 110 & 99.1$^{\dagger}$ & 85.5 & 44.5 & 71.8 & 76.4 & 66.4 & 89.1 & 81.8 \\
Conditional & 20 & 0.0 & 45.0 & 0.0 & 35.0 & 50.0 & 30.0 & 50.0 & 45.0 \\
Uncertainty & 40 & 0.0 & 35.0 & 12.5 & 35.0 & 32.5 & 37.5 & 52.5 & 50.0 \\
Family History & 30 & 0.0 & 56.7 & 3.3 & 43.3 & 43.3 & 43.3 & 50.0 & 56.7 \\
Sequence & 40 & 0.0 & 82.5 & 7.5 & 60.0 & 45.0 & 55.0 & 10.0 & 62.5 \\
Current State & 50 & 0.0 & 18.0 & 26.0 & 32.0 & 32.0 & 54.0 & 30.0 & 70.0 \\
Duration & 30 & 0.0 & 43.3 & 33.3 & 66.7 & 63.3 & 56.7 & 46.7 & 63.3 \\
Historical & 50 & 0.0 & 62.0 & 6.0 & 48.0 & 46.0 & 42.0 & 12.0 & 60.0 \\
Change & 30 & 0.0 & 56.7 & 10.0 & 36.7 & 23.3 & 20.0 & 6.7 & 96.7 \\
\midrule
\textbf{Overall} & \textbf{400} & \textbf{27.2}$^{\dagger}$ & \textbf{59.2} & \textbf{21.8} & \textbf{52.0} & \textbf{50.8} & \textbf{50.0} & \textbf{46.2} & \textbf{68.5} \\
\bottomrule
\end{tabular}
\caption{Complete ClinicalBench per-category accuracy (\%, Claude Opus 4.6, keyword evaluator v2) for all conditions, computed over the full 400-item set. The change-excluded keyword endpoint (sensitivity comparator; the primary endpoint is the leave-author-out McNemar in Section~\ref{sec:results:primary_endpoint}) excludes the 30 change-category items \emph{and} 8 experiencer-attribution-defective items (Appendix~\ref{app:bugged_items}), yielding $n=362$; per-category counts above are for the full set. C2b (Contriever dense retrieval) scores comparably to C2 (TF-IDF), confirming the retrieval method does not drive the C2$\to$C4g gap. C4 adds assertion metadata to C3 without intent routing; it helps assertion-sensitive categories (negation: $+22.7\pp$) but hurts temporal categories (historical: $-30\pp$, sequence: $-45\pp$). Only with intent routing (C4g) does the full system reach 68.5\%. $^{\dagger}$C7 is reported as ``--- (evaluator artifact, semantic 0\%)'' in Table~\ref{tab:clinicalbench_ablation}; raw 27.2\% reflects template refusals coincidentally matching negation keywords (see Section~\ref{sec:results:clinicalbench}).}
\label{tab:full_percategory}
\end{table}

\subsection{Extension Condition Per-Category Results}
\label{app:extension_categories}

\begin{table}[ht]
\centering
\small
\caption{Per-category accuracy (\%) for extension conditions ($n=240$, frozen evaluator + canonical questions\_v2 gold). C1b: discharge summary only. C4g+full: intent-aware KG-RAG with all clinical notes. The extension subset spans 4 categories: negation ($n=110$), current state ($n=50$), historical ($n=50$), duration ($n=30$).}
\label{tab:extension_categories}
\begin{tabular}{@{}l cc cc cc cc@{}}
\toprule
& \multicolumn{2}{c}{Opus} & \multicolumn{2}{c}{Qwen3.5} & \multicolumn{2}{c}{GPT-OSS} & \multicolumn{2}{c}{MedGemma} \\
\cmidrule(lr){2-3} \cmidrule(lr){4-5} \cmidrule(lr){6-7} \cmidrule(lr){8-9}
Category & C1b & C4g+ & C1b & C4g+ & C1b & C4g+ & C1b & C4g+ \\
\midrule
Negation & 82.7 & 92.7 & 82.7 & 86.4 & 84.5 & 85.5 & 80.9 & 85.5 \\
Current st. & 14.0 & 24.0 & 38.0 & 38.0 & 56.0 & 44.0 & 20.0 & 28.0 \\
Historical & 52.0 & 74.0 & 36.0 & 56.0 & 22.0 & 14.0 & 44.0 & 44.0 \\
Duration & 46.7 & 56.7 & 40.0 & 76.7 & 43.3 & 86.7 & 36.7 & 76.7 \\
\midrule
\textbf{Overall} & \textbf{57.5} & \textbf{70.0} & \textbf{58.3} & \textbf{68.8} & \textbf{60.4} & \textbf{62.1} & \textbf{55.0} & \textbf{63.8} \\
\bottomrule
\end{tabular}
\end{table}

All four models gain on the extension endpoint, with duration showing the largest gain across architectures (Opus $+10\pp$, Qwen3.5 $+36.7\pp$, GPT-OSS $+43.4\pp$, MedGemma $+40\pp$), indicating cross-admission temporal recall benefits substantially from structured graph retrieval. Opus and Qwen3.5 also gain on historical (Opus $+22\pp$, Qwen3.5 $+20\pp$); GPT-OSS shows a slight decline on historical ($-8\pp$) and current state ($-12\pp$), consistent with smaller-model difficulty stably extracting cross-admission state. The reference-answer defects in historical (Section~\ref{sec:results:physician}) attenuate measured gains on this category across all four models.

\section{Physician Adjudication Protocol and Summary}
\label{app:human}
\label{app:physician_protocol}

\paragraph{Design.} A single board-certified emergency physician (A.S.) independently scored ClinicalBench answers under two conditions (C1 and C4g) using a five-dimensional rubric (reference-answer correctness, model-answer correctness, auto-score fairness, clinical safety, clinical utility).
The reviewer was blinded to condition assignment (conditions randomized as A/B), so this should be interpreted as a blinded internal expert audit rather than independent external validation.
The paired adjudication covers 120 unique questions $\times$ 2 conditions.
Item-level reference-answer and safety summaries use the full adjudication record set, which includes one repeated blinded-condition review and therefore sums to 241 records.
Free-text physician notes were recorded for 165/241 items (68.5\%).
Reviews were conducted using a custom scoring interface with access to full de-identified MIMIC-IV discharge summaries.
Condition assignment was randomized and approximately balanced across blinded labels; this verifies allocation balance, not successful deblinding prevention, because C4g outputs can be more structured than C1 outputs.

\paragraph{Endpoints.} (1)~Human--keyword evaluator concordance rate (fraction of automated scores confirmed by physician).
(2)~Physician-rated C4g vs.\ C1 accuracy, safety, and utility.
(3)~Reference-answer error rate and defect taxonomy.

\paragraph{Reference-answer error analysis.}
Of the 241 audited records, only 44\% of v2 reference answers were rated fully correct; 29.5\% were outright wrong and 19.5\% needed revision.
Errors are concentrated in change (0\% correct---NLP conflated inpatient orders with discharge medications), historical (56.7\% wrong---``history of X'' misclassified as resolved), and uncertainty (37.5\% wrong---causal uncertainty conflated with existential uncertainty).
Because the detected defects are question-level rather than condition-specific, they are less likely to reverse the direction of the C1--C4g comparison. They do, however, materially affect absolute accuracies and some category-level magnitudes, especially for change and historical questions.

\begin{table}[ht]
\centering
\small
\caption{Reference-answer quality by category (physician adjudication of v2 reference set). Defect rate = fraction of reference answers rated less than fully correct. Categories with the highest defect rates drive most evaluator disagreements.}
\label{tab:gold_quality}
\begin{tabular}{@{}lrcc@{}}
\toprule
Category & $n$ & Correct & Defect Rate \\
\midrule
Change & 31 & 0.0\% & 100.0\% \\
Historical & 30 & 36.7\% & 63.3\% \\
Uncertainty & 40 & 40.0\% & 60.0\% \\
Current state & 30 & 43.3\% & 56.7\% \\
Negation & 10 & 50.0\% & 50.0\% \\
Sequence & 40 & 50.0\% & 50.0\% \\
Conditional & 20 & 55.0\% & 45.0\% \\
Duration & 20 & 65.0\% & 35.0\% \\
Family history & 20 & 70.0\% & 30.0\% \\
\midrule
\textbf{Overall} & \textbf{241} & \textbf{44.0\%} & \textbf{56.0\%} \\
\bottomrule
\end{tabular}
\end{table}

\paragraph{Auto-evaluator concordance.}
The keyword evaluator agreed with physician judgment in 54.2\% of cases ($n=240$).
When it disagreed, it was overwhelmingly too strict: 97 false negatives (40.4\%) vs.\ 13 false positives (5.4\%), a 7.5:1 strict-to-lenient ratio.
The majority of evaluator errors (63\% of false negatives) trace to reference-answer errors rather than evaluator logic.

\paragraph{Disclosure and limitations.} The reviewing physician (A.S.) is the author and system designer (see Section~\ref{sec:bench:physician} for primary disclosure).
Single-reviewer design limits inter-rater reliability assessment; three-rater external validation results (including two independent external physicians) are reported in Section~\ref{sec:results:external}.

\paragraph{Status.} This protocol was designed before full result synthesis.
Full adjudication results ($n=120$ paired) are reported in Section~\ref{sec:results:physician}; complete per-category results in Appendix~\ref{app:adjudication_full}.

\section{Reproducibility Details}
\label{app:reproducibility}

\paragraph{Models.} ClinicalBench primary ablation: Claude Opus 4.6 (claude-opus-4-6, keyword evaluator v2).
Cross-model: MedGemma 27B (alibayram/medgemma:27b, 4-bit GGUF), GPT-OSS 20B (4-bit GGUF), Qwen3.5 35B, Gemma~4 31B (\texttt{gemma4:31b}, 4-bit GGUF, \texttt{num\_predict=2048}), and MedGemma 1.5 4B (\texttt{alibayram/medgemma15:4b}, 4-bit GGUF, \texttt{num\_predict=2048}, with \texttt{<unused94>}/\texttt{<unused95>} stop tokens) via Ollama; same evaluator.
SliceBench: Claude Sonnet 4.5 (claude-sonnet-4-5-20250929) answers, Opus 4.6 judges.
Temperature 0 throughout; 4-bit quantization introduces GPU non-determinism ($\pm 10\pp$ run-to-run for MedGemma), controlled via within-run paired comparisons.

\paragraph{Evaluation.} The deterministic keyword evaluator performs exact word-boundary matching (regex \verb|\bKEYWORD\b|) against reference-answer assertion and temporal keywords.
Per-category keyword sets: negation (``no,'' ``denies,'' ``absent,'' ``negative''), uncertainty (``possible,'' ``suspected,'' ``may,'' ``likely,'' ``consider''), temporal (``before,'' ``after,'' ``during,'' ``changed,'' ``new''), plus category-specific terms.
Evidence preambles (echoed graph context) are stripped before matching.

\paragraph{Retrieval.} C2's vanilla RAG uses TF-IDF over chunked clinical notes (512-token chunks, 64-token overlap), retrieving the top-5 chunks by cosine similarity.
C2b replaces TF-IDF with Contriever~\cite{izacard2022contriever} dense embeddings (256-token chunks, 64-token overlap, top-$k$ up to 6{,}000 characters); C2b scores 50.8\% ($\Delta = -1.2\pp$ vs.\ C2; $p = 0.62$), confirming the retrieval method does not explain the C2$\to$C4g gap.
These are the closest analogs to existing medical RAG systems (DoctorRAG, MedRAG) on patient-level cross-admission data.
ClinicalBench conditions map to SliceBench: C1$\!\approx\!$B0, C2$\!\approx\!$B2, C4g$\!\approx\!$B3, C4g+full$\!\approx\!$B4.

\paragraph{Bootstrap.} $n=2{,}000$ resamples, seed 42, BCa method~\cite{efron1993bootstrap}.
Patient-level cluster bootstrap (resampling the 43 patients with replacement, including all their questions per draw) is the inferential method for keyword-endpoint CIs. Question-level resampling is reported as a secondary sensitivity analysis. (Caveat: with $n=43$ clusters this is at the low end of cluster-bootstrap reliability; cf.\ Cameron \& Miller~\cite{cameron2015cluster}; the substantive primary endpoint is the leave-author-out paired exact McNemar in Section~\ref{sec:results:primary_endpoint}, which does not depend on cluster bootstrap.)
Patient-level CIs are slightly wider but all reported endpoints remain significant (Table~\ref{tab:stats_summary}).

\paragraph{McNemar's test.} Bootstrap CIs are supplemented with CIs with McNemar's test for paired nominal data, comparing discordant pairs (C1 wrong/C4g right vs.\ C1 right/C4g wrong). All three pairwise comparisons are significant: C1 vs.\ C4g ($\chi^2 = 155.8$, $p < 10^{-6}$; discordant: 18 vs.\ 204), C1 vs.\ C3 ($\chi^2 = 73.8$, $p < 10^{-6}$; 30 vs.\ 143), and C3 vs.\ C4g ($\chi^2 = 47.2$, $p < 10^{-10}$; 20 vs.\ 93).
On the hard cross-admission subset (secondary endpoint, $n=122$: change $\cup$ current\_state $\cup$ historical, after excluding the 8 experiencer-attribution-defective items in Appendix~\ref{app:bugged_items}), C4g$_{\text{oracle}}$ vs.\ C1 yields $\Delta = +57.4\pp$, $p < 10^{-15}$.

\paragraph{BH-FDR adjustment for cross-model contrasts.} Across the six cross-model C4g$_{\text{oracle}}$ vs.\ C1 McNemar tests reported in Table~\ref{tab:stats_summary} and Section~\ref{sec:results:crossmodel}, Benjamini--Hochberg false-discovery-rate adjustment yields q-values: Opus $q=1.84\times10^{-32}$, GPT-OSS $q=3.14\times10^{-27}$, MedGemma 27B $q=1.22\times10^{-17}$, Gemma~4 31B $q=3.81\times10^{-17}$, Qwen3.5 35B $q=2.00\times10^{-11}$, MedGemma 1.5 4B $q=9.32\times10^{-11}$. BH-FDR applied across the six cross-model contrasts; all $q<10^{-10}$.
BH-FDR is reported here. The more conservative Benjamini--Yekutieli (2001) procedure (which holds under arbitrary dependence, including PRDS violations of paired McNemars on overlapping items) yields equivalent conclusions: BY adjustment factor $2.45\times$ over BH on $m=6$ cross-model contrasts, with all BY q-values $<2.5\times10^{-10}$.

\begin{table}[ht]
\centering
\small
\caption{Statistical summary: patient-level (cluster) and question-level BCa bootstrap 95\% CIs for the change-excluded keyword sensitivity endpoint and other secondary contrasts. The author-blind primary endpoint is the leave-author-out paired exact McNemar in Section~\ref{sec:results:primary_endpoint}; the keyword endpoint is change-excluded ($n=362$, after excluding the 8 experiencer-attribution-defective items in Appendix~\ref{app:bugged_items}) keyword C4g vs.\ C1.}
\label{tab:stats_summary}
\begin{tabular}{@{}lcccc@{}}
\toprule
Comparison & $\Delta$ & Question CI & Patient CI & McNemar $p$ \\
\midrule
C4g$-$C1 (full-set oracle, secondary) & $+46.5\pp$ & \ci{+40.5\pp}{+51.7\pp} & \ci{+40.8\pp}{+52.4\pp} & $< 10^{-6}$ \\
C2b$-$C2 (dense vs.\ keyword) & $-1.2\pp$ & \ci{-6.2\pp}{+3.5\pp} & \ci{-5.6\pp}{+3.5\pp} & $0.62$ \\
C4g$-$C2b (KG-RAG vs.\ dense) & $+17.5\pp$ & \ci{+11.8\pp}{+22.8\pp} & \ci{+12.4\pp}{+22.7\pp} & $< 10^{-6}$ \\
C3$-$C1 (retrieval) & $+28.2\pp$ & \ci{+22.2\pp}{+34.2\pp} & \ci{+23.1\pp}{+33.3\pp} & $< 10^{-6}$ \\
C4$-$C3 (assertions) & $-3.8\pp$ & \ci{-10.8\pp}{+1.8\pp} & \ci{-10.7\pp}{+3.1\pp} & $0.26$ \\
C4g$-$C4 (routing) & $+22.0\pp$ & \ci{+15.5\pp}{+27.8\pp} & \ci{+14.8\pp}{+29.4\pp} & $< 10^{-10}$ \\
C4g$-$C3 (both) & $+18.2\pp$ & \ci{+13.2\pp}{+22.8\pp} & \ci{+13.7\pp}{+23.1\pp} & $< 10^{-10}$ \\
Hard cross-admission ($n=122$) & $+57.4\pp$ & \ci{+47.4\pp}{+65.9\pp} & \ci{+46.8\pp}{+68.6\pp} & $< 10^{-15}$ \\
\bottomrule
\end{tabular}
\end{table}

\paragraph{Data.} De-identified MIMIC-IV clinical notes accessed under a PhysioNet Credentialed Health Data Use Agreement (v3.1).

\paragraph{PhysioNet DUA compliance.} The HuggingFace release of the ClinicalBench artifact (DOI \href{https://doi.org/10.57967/hf/8549}{10.57967/hf/8549}) contains: (a) question text (authored from MIMIC-IV charts but rephrased; not raw note text), (b) reference answers (paraphrased clinical findings; not direct note quotations), and (c) raw model predictions (commit \texttt{da5f5b1} stripped raw note excerpts from adjudication items prior to public release). MIMIC-IV note text remains exclusively under PhysioNet credentialed access. Any user replicating the system requires separate PhysioNet authorization.

\paragraph{Released artifacts.} ClinicalBench questions, reference answers (v1 and v2), raw model predictions for all evaluated conditions, the deterministic keyword evaluator, and physician adjudication data are publicly released (Hugging Face DOI \href{https://doi.org/10.57967/hf/8549}{10.57967/hf/8549}). The full EpiKG application stack (graph construction, intent-aware routing implementation, retrieval algorithm) is NOT released in this work. Application-code release is planned for follow-on work; reviewers and downstream users should treat the system contribution here as a probe rather than a reproducible architectural artifact.

\paragraph{Compute.} ClinicalBench Opus C1/C3/C4g: ${\sim}22$min each (API); Opus C6: ${\sim}3.5$h (API); C7: $<1$min (no LLM); MedGemma 27B C1--C4g: ${\sim}3.6$h (single GPU); MedGemma 27B C6: ${\sim}1.5$h; Gemma~4 31B C1/C4g: ${\sim}10$h (Apple Silicon, Ollama); MedGemma 1.5 4B C1/C4g: ${\sim}1.5$h (Apple Silicon, Ollama); SliceBench: ${\sim}2$h (API).

\paragraph{Evaluator evolution.} The keyword evaluator underwent three versions: v0 (substring matching), v1 (word-boundary matching + evidence preamble stripping), and v2 (+ abstention detection gate + domain-specific keyword requirements for sequence and change). The v1 evaluator awarded false positives when models responded with ``insufficient information in the notes'' because negation and temporal keywords in the refusal text matched reference-answer patterns. The v2 evaluator adds an abstention detection layer: answers matching abstention patterns (e.g., ``cannot determine,'' ``not mentioned,'' ``insufficient evidence'') are scored as incorrect unless they contain clinical claim patterns (e.g., ``patient does not,'' ``denies''). Additionally, v2 requires sequence answers to contain ordering keywords (``first,'' ``then,'' ``before'') and change answers to contain change keywords (``added,'' ``removed,'' ``discontinued''), preventing term-overlap-only false positives. The duration category's minimum-0.5 score floor for matching duration keywords was removed.
The v1$\to$v2 transition reduced C1 accuracy (from ${\sim}50\%$ to 21.8\% for Opus) by correctly classifying abstention responses as incorrect; C4g accuracy decreased modestly (from ${\sim}76\%$ to 68.5\%) because the model abstains less frequently when retrieval context is provided.
All main-text numbers use v2.

\begin{sloppypar}
\paragraph{Reproducibility package.} A standalone reproducibility package is included in the supplementary materials (\texttt{epikg-benchmark/}) and available at \url{https://huggingface.co/datasets/alexstinard/epikg-clinicalbench}. It contains all 400 ClinicalBench questions with reference answers, raw model predictions for all LLM-based conditions (Opus C1/C2/C3/C4/C4g/C6, MedGemma 27B C1/C4g, GPT-OSS C1/C4g, Qwen C1/C4g, Gemma~4 31B C1/C4g, MedGemma 1.5 4B C1/C4g), scored outputs for the deterministic baseline (C7), and the keyword evaluator v2 with abstention detection. The evaluator itself requires only Python~3.10+ with no external dependencies; \texttt{reproduce.py} additionally requires NumPy and SciPy for bootstrap CIs. All ClinicalBench accuracy numbers in Tables~\ref{tab:clinicalbench_ablation}, \ref{tab:crossmodel_categories}, and~\ref{tab:crossmodel} are reproducible from this package. MIMIC-IV patient identifiers are included so that PhysioNet-credentialed reviewers can trace predictions back to source clinical notes.
ClinicalBench is publicly available at \url{https://huggingface.co/datasets/alexstinard/epikg-clinicalbench} with Croissant metadata for machine-readable dataset discovery.
\end{sloppypar}

\paragraph{Results provenance.} Table~\ref{tab:provenance} maps each reported result to its source checkpoint file.
MedGemma 27B C4g has 2 empty predictions (timeouts); these are scored as incorrect ($n=400$).
All cross-model C1 and C4g results (Table~\ref{tab:crossmodel}) are from the same system snapshot (February 2026), except Gemma~4 31B (April 2026, added after the initial cross-model batch) and MedGemma 1.5 4B (April 2026, after fixing a checkpoint truncation bug and adding \texttt{<unused94>}/\texttt{<unused95>} Gemma~3 stop tokens in the Ollama Modelfile); intermediate ablation conditions (C2/C3/C4/C6) for non-Opus models were collected in a later batch (March 2026) after system updates and are included in the reproducibility package for completeness but are not used in any paper table.

\begin{table}[ht]
\centering
\small
\caption{Results provenance: checkpoint files for each condition$\times$model combination. All scored with keyword evaluator v2.}
\label{tab:provenance}
\begin{tabular}{@{}llllr@{}}
\toprule
Condition & Model & Checkpoint file & $n$ \\
\midrule
C1 & Opus 4.6 & \texttt{opus/C1\_llm\_alone.jsonl} & 400 \\
C2 & Opus 4.6 & \texttt{opus/C2\_vanilla\_rag.jsonl} & 400 \\
C2b & Opus 4.6 & \texttt{opus/C2b\_dense\_rag.jsonl} & 400 \\
C3 & Opus 4.6 & \texttt{opus/C3\_kg\_rag.jsonl} & 400 \\
C4 & Opus 4.6 & \texttt{opus/C4\_epistemic\_kg\_rag.jsonl} & 400 \\
C4g & Opus 4.6 & \texttt{opus/C4g\_intent\_aware.jsonl} & 400 \\
C6 & Opus 4.6 & \texttt{opus/C6\_long\_context.jsonl} & 400 \\
C7 & --- & \texttt{opus/C7\_deterministic.jsonl} & 400 \\
\midrule
C1 & MedGemma 27B & \texttt{medgemma/C1\_llm\_alone.jsonl} & 400 \\
C4g & MedGemma 27B & \texttt{medgemma/C4g\_intent\_aware.jsonl} & 400* \\
\midrule
C1 & GPT-OSS 20B & \texttt{gptoss/C1\_llm\_alone.jsonl} & 400 \\
C4g & GPT-OSS 20B & \texttt{gptoss/C4g\_intent\_aware.jsonl} & 400 \\
\midrule
C1 & Qwen3.5 35B & \texttt{qwen35/C1\_llm\_alone.jsonl} & 400 \\
C4g & Qwen3.5 35B & \texttt{qwen35/C4g\_intent\_aware.jsonl} & 400 \\
\midrule
C1 & Gemma~4 31B & \texttt{gemma4/C1\_llm\_alone.jsonl} & 400 \\
C4g & Gemma~4 31B & \texttt{gemma4/C4g\_intent\_aware.jsonl} & 400 \\
\midrule
C1 & MedGemma 1.5 4B & \texttt{medgemma15/C1\_llm\_alone.jsonl} & 400 \\
C4g & MedGemma 1.5 4B & \texttt{medgemma15/C4g\_intent\_aware.jsonl} & 400 \\
\bottomrule
\multicolumn{4}{@{}l}{\footnotesize *MedGemma 27B C4g: 2 timeouts; empty answers scored as incorrect ($n=400$).}
\end{tabular}
\end{table}

\subsection{Infrastructure Bugs Discovered and Fixed During Evaluation}
\label{app:infrastructure_bugs}

During final review, three infrastructure issues were discovered and corrected that affected data integrity in earlier runs. They are documented here in the interest of methodological transparency; all three have since been fixed, and the residual impact on reported numbers is described for each.

\paragraph{Bug 1: Checkpoint serialization truncation (500-character limit).}
A line in \texttt{qa\_experiment\_executor.py} wrote \texttt{predicted\_answer[:500]} to checkpoint files, silently truncating every saved answer to 500 characters. Original in-run scoring (the legacy internal evaluator) operated on full answers, but the frozen-evaluator rescoring---which is the source of all numbers in this paper---operated on the truncated strings. The truncation rate scaled with per-model answer verbosity: GPT-OSS $\approx 0.2\%$ of answers affected, MedGemma 27B $\approx 0.5$--$2.8\%$ (all negligible, $<\!1\pp$ impact on frozen rescore), Qwen~3.5 $\approx 0.8$--$4.2\%$ ($<\!1\pp$), Opus~4.6 $\approx 3.5$--$16\%$ (estimated $1$--$3\pp$ downward bias on C4g oracle), and MedGemma~1.5 4B $\approx 7.5$--$35.5\%$ (estimated $2$--$3\pp$ downward bias). Because the bug was applied uniformly across models and conditions, relative rankings are preserved; absolute scores for the most verbose models (Opus, MedGemma~1.5) would shift slightly upward if fully rerun. The slice was removed from the \texttt{[:500]} slice and reran MedGemma~1.5 end-to-end (the most affected model) to obtain clean data; other models retain their original checkpoint data with this documented downward bias.

\paragraph{Bug 2: MedGemma 1.5 special-token duplication.}
MedGemma~1.5 4B (Gemma~3 architecture) emitted the special tokens \texttt{<unused94>} and \texttt{<unused95>} mid-generation, causing the model to regenerate the same answer twice within a single response. The frozen evaluator then scored on the concatenated duplicate text, effectively giving the model two independent chances to match reference keywords. Manual inspection of per-question outputs showed that $26\%$ of MedGemma~1.5 C4g oracle answers contained these tokens, artificially inflating frozen-evaluator scores by approximately $2$--$5\pp$. Only MedGemma~1.5 was affected; the tokens are Gemma~3 specific. The fix was trivial---adding \texttt{<unused94>} and \texttt{<unused95>} to the Ollama Modelfile stop-token list---but the benchmarking impact was significant. All MedGemma~1.5 numbers in this paper were collected after the fix was applied and the model was rerun across all three reported conditions.

\paragraph{Bug 3: Ollama silently discards Qwen~3.5 repetition penalties in thinking mode.}
Qwen~3.5 35B is designed as an extended-reasoning model; its model card explicitly recommends \texttt{presence\_penalty=1.5} to prevent pathological repetition loops during chain-of-thought generation. However, Ollama silently discards \texttt{repeat\_penalty}, \texttt{presence\_penalty}, and \texttt{frequency\_penalty} options when Qwen is used with \texttt{think: true} (Ollama issue \#14493). At \texttt{temperature=0} the model then enters infinite reasoning loops (e.g., ``Wait, I'll write: `X'. Okay, I'll write: `Y'. Wait, I'll write: `X'\ldots'') and consumes the entire token budget without producing any visible content. The issue was verified by testing eight sampling configurations through the Ollama API (\texttt{repeat\_penalty} from $1.3$ to $1.8$, \texttt{frequency\_penalty=0.5}, \texttt{presence\_penalty=0.5}, \texttt{temperature} $\in \{0.3, 0.5\}$, \texttt{mirostat=2}, and stacked combinations); all eight produced identical $29{,}083$-character thinking output with empty content, confirming that Ollama is not forwarding these options to Qwen's thinking generation. Related open Ollama issues include \#14493 (Qwen~3.5 tool calling non-functional and repetition penalties silently ignored), \#14421 (\texttt{qwen3.5:35b} looping), \#10976 (thinking + tools + qwen3 $\Rightarrow$ empty output), \#14716 (\texttt{qwen3.5} vision output routed to thinking field), and \#10927 (LLM stuck in infinite loop of thinking).

As a workaround thinking was disabled (\texttt{think: false}) for all Qwen runs reported in this paper. This produces direct answers but may not reflect Qwen's optimal reasoning performance. Qwen's reported numbers should therefore be interpreted as ``Qwen~3.5 35B without reasoning enabled'' rather than ``Qwen~3.5 35B at full capability.'' This workaround has a visible downstream effect on Qwen's oracle-vs-keyword comparison: a $-1.2\pp$ oracle inversion at the aggregate level that does not match the pattern of the other five tested models. Per-category analysis shows Qwen benefits substantially from oracle routing on historical questions ($+16\pp$) but suffers on \texttt{current\_state} ($-10\pp$) and conditional ($-15\pp$) categories. One hypothesis is that Qwen's constrained non-thinking mode interacts poorly with oracle's focused \texttt{current\_state} retrieval, which discards historical context that non-thinking Qwen appears to rely on. Whether this inversion would persist with thinking enabled---via vLLM or Qwen's native DashScope API, which do honor repetition penalties---is an open question is left for future work.

\subsection{Experiencer-Attribution Defective Items}
\label{app:bugged_items}

During post-hoc verification, 8 items were identified whose source \texttt{section="Family History"} but whose gold \texttt{expected\_answer} asserts the disease as a current or historical condition of the patient. The defect arises from the upstream NLP pipeline mis-propagating an experiencer flag during gold-answer generation: the section tag was correctly retained but the answer string nonetheless described the patient. These items are excluded from the change-excluded keyword endpoint ($n=362$, sensitivity comparator). The headline impact of removing them is $+40.0\pp \to +39.5\pp$ on that endpoint. They remain in the released v2 gold standard for transparency.

\begin{table}[ht]
\centering
\small
\caption{Eight experiencer-attribution-defective items excluded from the change-excluded keyword sensitivity endpoint. All have \texttt{section="Family History"} but gold answers describing the patient.}
\label{tab:bugged_items}
\begin{tabular}{@{}llll@{}}
\toprule
QID & Category & Disease & Bug \\
\midrule
\texttt{bench\_b\_current\_state\_0a964177} & current\_state & Lung cancer & gold says current; FH \\
\texttt{bench\_b\_current\_state\_351fb38e} & current\_state & Lung cancer & gold says current; FH \\
\texttt{bench\_b\_current\_state\_7080eb03} & current\_state & Hypertension & gold says current; FH \\
\texttt{bench\_b\_current\_state\_96ef4bd2} & current\_state & Breast cancer & gold says current; FH \\
\texttt{bench\_b\_current\_state\_9a0fed0b} & current\_state & Colon cancer & gold says active; FH \\
\texttt{bench\_b\_current\_state\_ab1c0783} & current\_state & Colon cancer & gold says current; FH \\
\texttt{bench\_b\_current\_state\_e84d9e91} & current\_state & Breast cancer & gold says current; FH \\
\texttt{bench\_b\_historical\_a701eaf4} & historical & Colon cancer & gold says historical; FH \\
\bottomrule
\end{tabular}
\end{table}

\subsection{Evaluator Polarity: Extended Disclosure}
\label{app:evaluator_polarity}

The keyword evaluator uses category-specific keyword lists and matches by word-boundary regex. For three categories---uncertainty, family\_history, and conditional---the rule is structurally:
\begin{quote}\small\texttt{is\_correct = \_has\_match(predicted\_lower, patterns)}\end{quote}
where \texttt{patterns} is the category-defining keyword list (e.g., \texttt{['if','conditional','pending','depending','only if']} for conditional). The gold \texttt{expected\_answer} is not consulted: any prediction containing a category keyword is scored correct.

For current\_state and historical, the rule matches keyword presence (\texttt{"current"}, \texttt{"active"}, \texttt{"present"} for current\_state; \texttt{"was"}, \texttt{"former"}, \texttt{"resolved"}, \texttt{"history"} for historical) without polarity check. Consequently, a prediction asserting \emph{``NOT FOUND IN CURRENT RECORDS''} matches the keyword \texttt{"current"} and is scored correct against a gold of ``currently active''---directionally opposed but lexically overlapping.

The implication is that C4g's structured-answer style (which routinely echoes the queried category, e.g., ``Current state: \ldots'') is mechanically advantaged over C1's abstention style (``insufficient information''), even when neither answer carries the right clinical content. This contributes to the keyword evaluator's measured 7.5:1 strict-vs-lenient asymmetry (Table~\ref{tab:evaluator_agreement}) and is the principal reason we treat the keyword evaluator as a deterministic reproducibility proxy rather than a substantive truth criterion. Physician adjudication and LLM-as-judge are the substantively interpretable evaluators; their deltas (Section~\ref{sec:results:physician}) are the comparisons we ask readers to weight.

\section{Assertion Category Definitions}
\label{app:assertions}

\begin{table}[ht]
\centering
\small
\begin{tabular}{lp{6.5cm}l}
\toprule
\textbf{Assertion} & \textbf{Definition} & \textbf{Example} \\
\midrule
\textsc{present} & Condition currently affirmed & ``has diabetes'' \\
\textsc{absent} & Condition explicitly negated & ``denies chest pain'' \\
\textsc{possible} & Condition suspected, not confirmed & ``possible pneumonia'' \\
\textsc{conditional} & Contingent on specific circumstances & ``if febrile, start antibiotics'' \\
\textsc{hypothetical} & Discussed as a scenario & ``would need dialysis if\ldots'' \\
\textsc{family\_history} & Attributed to a family member & ``mother had breast cancer'' \\
\textsc{historical} & Previously true, not necessarily current & ``former smoker'' \\
\bottomrule
\end{tabular}
\caption{The 7-value assertion taxonomy, extending i2b2 by separating \textsc{historical} from \textsc{family\_history}.}
\label{tab:assertion_defs}
\end{table}

\section{Temporal Relation Mapping}
\label{app:temporal}

The nine temporal relations $\mathcal{R}$ used on KG edges are derived from Allen's 13 canonical interval relations~\cite{allen1983temporal} by merging symmetric pairs.

\begin{table}[ht]
\centering
\small
\begin{tabular}{@{}lll@{}}
\toprule
\textbf{$\mathcal{R}$ value} & \textbf{Allen source(s)} & \textbf{Meaning} \\
\midrule
\textsc{Before} & Before, Meets & $A$ ends before $B$ starts \\
\textsc{After} & After, Met-by & $A$ starts after $B$ ends \\
\textsc{During} & During & $A$ entirely within $B$ \\
\textsc{Contains} & Contains & $A$ entirely contains $B$ \\
\textsc{Overlaps} & Overlaps, Overlapped-by & $A$ and $B$ partially overlap \\
\textsc{Starts} & Starts, Started-by & $A$ and $B$ share start time \\
\textsc{Finishes} & Finishes, Finished-by & $A$ and $B$ share end time \\
\textsc{Concurrent} & Equals & $A$ and $B$ have same interval \\
\textsc{Unknown} & --- & Relation undetermined \\
\bottomrule
\end{tabular}
\caption{Mapping from Allen's 13 interval relations to the 9 temporal relation values stored on KG edges.}
\label{tab:temporal_relations}
\end{table}

\section{Intent-Aware Routing Algorithm}
\label{app:routing}

The C4g intent classifier operates in two modes.
In the \emph{oracle} mode used for benchmark evaluation, the question's category metadata determines the intent directly.
In the \emph{keyword-only} mode used for deployment, a rule-based classifier infers intent from keyword patterns in the question text (\eg ``changed'', ``new since'' trigger \textsc{Change}; ``currently'', ``active problem'' trigger \textsc{Current\_State}).
The keyword classifier achieves 68\% overall accuracy but only 20.0\% on questions requiring targeted routing; per-category accuracy is reported in Table~\ref{tab:keyword_classifier}.
All benchmark results in the main text use oracle classification unless otherwise noted.
Algorithm~\ref{alg:c4g} details the full retrieval procedure.

\begin{algorithm}[ht]
\caption{Intent-Aware Retrieval (C4g)}\label{alg:c4g}
\KwIn{Question $q$, patient $\pi$}
\KwOut{Structured evidence $E$}
$\mathcal{C} \gets \textsc{ExtractConcepts}(q)$ \tcp*{NLP + OMOP enrichment}
$\iota \gets \textsc{ClassifyIntent}(q)$ \tcp*{$\in \{\textsc{Change}, \textsc{CurrSt}, \textsc{Hist}, \textsc{Default}\}$}
\uIf{$\iota = \textsc{Change}$}{
  Partition edges by admission: $\mathcal{E}_k \gets \{e \mid \texttt{hadm\_id}(e) = k\}$\;
  \ForEach{admission pair $(k, k')$ with $k < k'$}{
    $\mathcal{A} \gets \mathcal{C}_{k'} \setminus \mathcal{C}_k$; \quad
    $\mathcal{R} \gets \mathcal{C}_k \setminus \mathcal{C}_{k'}$; \quad
    $\mathcal{S} \gets \mathcal{C}_k \cap \mathcal{C}_{k'}$\;
  }
  $E_g \gets \textsc{FormatChange}(\mathcal{A}, \mathcal{R}, \mathcal{S})$\;
}
\uElseIf{$\iota = \textsc{CurrSt}$}{
  $E_g \gets \textsc{FilterEdges}(\pi, \mathcal{C},\; \tau_a\!=\!\textsc{Current} \lor \text{open validity})$\;
  Deduplicate by concept; emit ``\textsc{Not Found}'' for missing $c \in \mathcal{C}$\;
}
\uElseIf{$\iota = \textsc{Hist}$}{
  $E_g \gets \textsc{FilterEdges}(\pi, \mathcal{C},\; \tau_a\!=\!\textsc{Past})$\;
  Augment: concepts in earlier admissions but absent from latest $\to$ ``resolved''\;
}
\Else{
  $E_g \gets \textsc{BidirectionalBFS}(\pi, \mathcal{C},\; \text{hops}=2\text{--}3,\; c_{\min}=0.3)$\;
}
$E_d \gets \textsc{RetrieveDocuments}(\pi, \mathcal{C})$\;
\Return $\textsc{Compose}(E_g, E_d,\; \text{template}(\iota))$\;
\end{algorithm}

\subsection{Keyword-Only Intent Classifier Accuracy}
\label{app:keyword_classifier}

Table~\ref{tab:keyword_classifier} reports per-category accuracy of the keyword-only intent classifier alongside the oracle--keyword accuracy gap on Opus C4g.
The keyword classifier achieves 68\% overall classification accuracy but only 20.0\% on questions requiring targeted routing (historical: 0\%, family history: 0\%, current state: 34\%, change: 50\%).
Categories routed to \textsc{Default} (negation, uncertainty, conditional, duration, sequence) are unaffected by classification errors because both oracle and keyword paths use the same default BFS traversal.

\begin{table}[ht]
\centering
\small
\caption{Per-category keyword intent classifier accuracy and its impact on Opus C4g QA accuracy ($n=400$). ``Classifier acc.'' is the fraction of questions where the keyword classifier matches oracle intent. Categories marked ``Default'' are routed identically under both classifiers.}
\label{tab:keyword_classifier}
\begin{tabular}{@{}lrcccr@{}}
\toprule
Category & $n$ & Classifier Acc. & Oracle & Keyword & $\Delta$ \\
\midrule
change & 30 & 50.0\% & 96.7\% & 30.0\% & $-66.7\pp$ \\
current\_state & 50 & 34.0\% & 70.0\% & 66.0\% & $-4.0\pp$ \\
family\_history & 30 & 0.0\% & 56.7\% & 46.7\% & $-10.0\pp$ \\
historical & 50 & 0.0\% & 60.0\% & 60.0\% & $\phantom{-}0.0\pp$ \\
negation & 110 & Default & 81.8\% & 75.5\% & $-6.3\pp$ \\
sequence & 40 & Default & 62.5\% & 55.0\% & $-7.5\pp$ \\
uncertainty & 40 & Default & 50.0\% & 50.0\% & $\phantom{-}0.0\pp$ \\
conditional & 20 & Default & 45.0\% & 45.0\% & $\phantom{-}0.0\pp$ \\
duration & 30 & Default & 63.3\% & 70.0\% & $+6.7\pp$ \\
\midrule
\textbf{Overall} & \textbf{400} & 68.0\% & \textbf{68.5\%} & \textbf{60.2\%} & $\mathbf{-8.3\pp}$ \\
\bottomrule
\end{tabular}
\end{table}

\begin{figure*}[t]
\centering
\includegraphics[width=\textwidth]{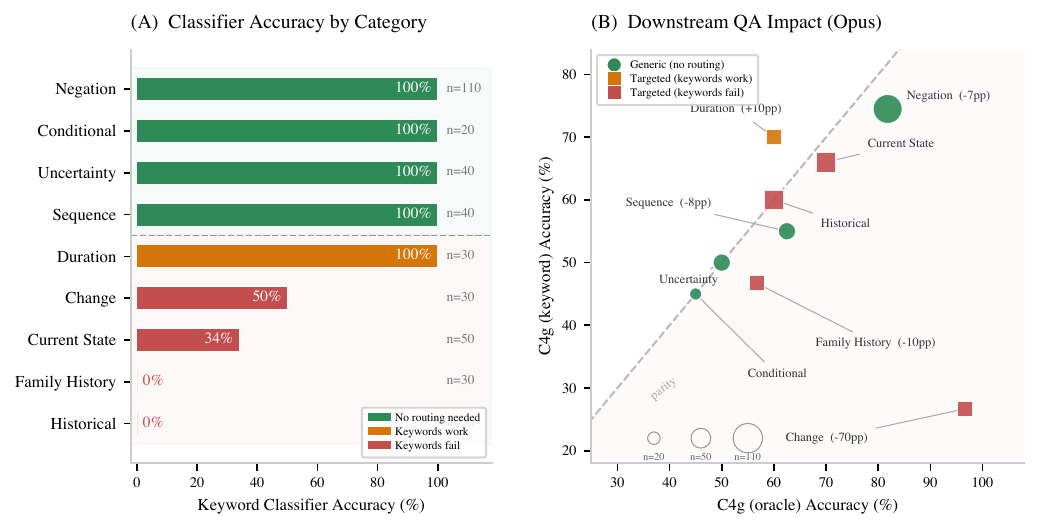}
\caption{Keyword intent classifier analysis (Opus). \textbf{(A)}~Per-category classifier accuracy: categories requiring no targeted routing (top, green) achieve 100\%; categories needing routing show 0--50\% keyword accuracy (bottom, red/orange). \textbf{(B)}~Downstream QA impact: categories near the parity line (dashed) have small oracle--keyword gaps; change ($-70\pp$) is the major outlier, while duration ($+10\pp$) benefits from keyword routing.}
\label{fig:classifier_detail}
\end{figure*}

\section{Bookend and Diagnostic Conditions}
\label{app:c5}

\paragraph{C5 (Full System).}
C5 extends C4 with guideline retrieval (1{,}202 sections) and clinical calculators (201, \eg CHA$_2$DS$_2$-VASc, MELD).
In a prior evaluator run, C5 scored below C4g, likely because non-relevant components (guidelines and calculators) dilute the context window without contributing to ClinicalBench question types.
C5 is excluded from the current ablation ladder because it conflates retrieval architecture with additional knowledge sources.

\paragraph{C6 (Long Context).}
All patient documents are concatenated chronologically and presented to the LLM.
For Opus (200K token window), all notes fit; for MedGemma (8K window), later documents are truncated.
C6 achieves 59.2\% (Opus) overall---well above C1 (21.8\%) and modestly below C4g$_\text{oracle}$ (68.5\%, $-9.3\pp$, $p = 0.001$).
The gap concentrates in current state (C6 18.0\% vs.\ C4g 70.0\%, $-52\pp$); on temporal categories C6 is comparable (historical 62.0\%, sequence 82.5\%) or weaker only modestly (duration 43.3\% vs.\ C4g 63.3\%).
This indicates brute-force long context handles factual recall and temporal retrieval reasonably well but underperforms structured retrieval on intent-sensitive queries, where the epistemic KG's explicit assertion typing distinguishes resolved from active conditions in ways implicit raw-text reading does not.

\paragraph{C7 (Deterministic KG).}
KG edges matching query concepts are returned directly without LLM reasoning.
C7 nominally scores 27.2\% overall, but this is an evaluator artifact: C7 returns template refusals (``No relevant knowledge graph edges found'') for $>$98\% of questions, and the word ``No'' in the template coincidentally matches negation-category keywords (negation: 99.1\%, all other categories: 0\%).
Semantic accuracy is 0\%, confirming that structured data without an LLM reasoner cannot answer clinical questions.

\section{Illustrative Retrieval Example}
\label{app:example}

To illustrate why intent-aware routing matters, consider a change question: \emph{``What medications changed between this patient's first and second admissions?''}

\paragraph{C1 (LLM alone).}
The model sees only the current note, which mentions metoprolol, atorvastatin, and ``discontinued lisinopril.'' Without prior admission records, it cannot determine what was \emph{added} vs.\ \emph{continued}, and may hallucinate prior medication lists.

\paragraph{C4g (intent-aware KG-RAG).}
The intent classifier routes to \textsc{Change} retrieval, which partitions KG edges by \texttt{hadm\_id} and computes set differences:
\begin{itemize}[nosep,leftmargin=1.5em]
\item \textbf{Added} (admission~2 only): atorvastatin~40mg [\textsc{Present}]
\item \textbf{Removed} (admission~1 only): lisinopril~10mg [\textsc{Present}$\to$\textsc{Absent}]
\item \textbf{Continued}: metoprolol~25mg [\textsc{Present}, both admissions]
\end{itemize}
The assertion labels (\textsc{Present}, \textsc{Absent}) disambiguate: ``discontinued lisinopril'' is not a current medication but a historical one whose status changed---exactly the distinction the epistemic schema preserves.

\section{LLM-as-Judge Concordance}
\label{app:llm_judge}

To complement the deterministic keyword evaluator, all were rescored for 800 predictions (400 questions $\times$ 2 conditions: C1 and C4g) using Claude Opus 4.6 as an LLM judge.
The judge prompt presents the question, reference answer, and system answer, and requests a score of 1 (correct), 0.5 (partially correct), or 0 (incorrect) with a one-sentence justification.
Table~\ref{tab:llm_judge} summarizes the concordance.

\begin{table}[ht]
\centering
\small
\caption{Evaluator concordance comparison. LLM judge scores $\geq 0.5$ treated as correct for binary comparison. Physician concordance computed on the $n=30$ pilot adjudication subset.}
\label{tab:llm_judge}
\begin{tabular}{@{}lcc@{}}
\toprule
\textbf{Metric} & \textbf{Keyword v2} & \textbf{LLM Judge} \\
\midrule
\multicolumn{3}{@{}l}{\textit{vs.\ Keyword evaluator ($n=800$)}} \\
Agreement & --- & 78.9\% \\
Cohen's $\kappa$ & --- & 0.572 \\
\midrule
\multicolumn{3}{@{}l}{\textit{vs.\ Physician ($n=30$)}} \\
Agreement & 46.7\% & 56.7\% \\
Too strict & 43.3\% & 36.7\% \\
Too lenient & 10.0\% & 6.7\% \\
\midrule
\multicolumn{3}{@{}l}{\textit{Condition-level accuracy}} \\
C1 accuracy & 21.8\% & 28.5\% \\
C4g accuracy & 68.5\% & 57.0\% \\
C4g$-$C1 $\Delta$ & $+46.8\pp$ & $+28.5\pp$ \\
\bottomrule
\end{tabular}
\end{table}

\paragraph{Key findings.}
The LLM judge achieves higher physician concordance (56.7\% vs.\ 46.7\%) and is less prone to false strictness (36.7\% vs.\ 43.3\%), confirming that the keyword evaluator's conservative bias inflates the measured delta.
Under the LLM judge, C1 rises from 21.8\% to 28.5\% (the judge credits clinically reasonable hedging that keyword matching penalizes), while C4g drops from 68.5\% to 57.0\% (the judge penalizes partial answers that pass keyword matching via term overlap).
The resulting C4g$-$C1 delta under LLM judge ($+28.5\pp$) is lower than the keyword delta ($+46.8\pp$) but remains substantial and directionally consistent.
The per-condition asymmetry---C1 gains, C4g loses---is consistent with the physician finding that the keyword evaluator disproportionately penalizes C1 abstentions.

\paragraph{Per-category analysis.}
The LLM judge is particularly stricter on change (C4g: 62\% mean score vs.\ 100\% keyword) because it penalizes partial medication lists that contain correct change keywords but miss specific drugs.
Conversely, the judge is more generous on historical (C4g: 62\% vs.\ 60\% keyword) and negation (C4g: 79\% vs.\ 81\% keyword), where its semantic understanding better captures correct answers that use varied phrasing.

\paragraph{Limitations.}
Using the same model family (Opus) for both answering and judging introduces potential self-preference bias.
The judge also shows moderate agreement with the keyword evaluator ($\kappa = 0.572$), indicating they capture partially overlapping but distinct aspects of answer quality.

\section{Assertion Classifier Evaluation}
\label{app:assertion_eval}

The rule-based assertion classifier uses 122 calibrated trigger patterns extending the NegEx~\cite{chapman2001negex} and ConText~\cite{harkema2009context} frameworks to a 7-class taxonomy (Table~\ref{tab:assertion_defs}).
Table~\ref{tab:assertion_inventory} summarizes the pattern inventory and confidence ranges by category.

\begin{table}[ht]
\centering
\small
\caption{Assertion trigger pattern inventory. Confidence ranges reflect per-trigger calibration.}
\label{tab:assertion_inventory}
\begin{tabular}{@{}lrcl@{}}
\toprule
Category & Patterns & Confidence & Example triggers \\
\midrule
Absent & 31 & 0.85--0.98 & ``no evidence of,'' ``denies,'' ``ruled out'' \\
Uncertain & 32 & 0.35--0.75 & ``possible,'' ``likely,'' ``consistent with'' \\
Present & 27 & 0.85--0.98 & ``confirmed,'' ``diagnosed with,'' ``positive for'' \\
Hypothetical & 11 & 0.25--0.35 & ``risk of,'' ``screening for,'' ``prophylaxis'' \\
Family history & 8 & 0.85--0.95 & ``family history of,'' ``maternal history'' \\
Conditional & 7 & 0.20--0.40 & ``if,'' ``contingent on,'' ``depending on'' \\
Historical & 6 & 0.75--0.90 & ``history of,'' ``former,'' ``remote history'' \\
\midrule
\textbf{Total} & \textbf{122} & 0.20--0.98 & \\
\bottomrule
\end{tabular}
\end{table}

\paragraph{Corpus statistics.}
Across the 43 ClinicalBench patients, the knowledge graph contains 3{,}943 edges linked to clinical facts.
Of these, 618 (15.7\%) carry non-present assertions: 442 absent (11.2\%), 94 possible (2.4\%), 71 historical (1.8\%), 8 conditional (0.2\%), and 3 hypothetical (0.1\%).
At the mention level, 1{,}428 of 12{,}379 mentions (11.5\%) are non-present.
This non-present fraction $f_{\textit{np}} = 0.157$ provides the empirical bound from Corollary~\ref{cor:faithfulness}: an assertion-blind pipeline cannot exceed $1 - f_{\textit{np}} = 84.3\%$ assertion-faithful accuracy on concepts where negated or uncertain mentions are present.

\paragraph{Intrinsic evaluation.}
A stratified sample of 189 mentions from the 43 ClinicalBench patients, stratified by predicted assertion type (50 present, 51 absent, 15 possible, 28 historical, 19 conditional, 4 hypothetical, 10 family history), and a physician (A.S.) annotated reference-standard labels in a blinded, randomized review.
Table~\ref{tab:assertion_intrinsic} reports per-class precision, recall, and F1.

\begin{table}[ht]
\centering
\small
\caption{Intrinsic assertion classifier evaluation on 189 physician-annotated
  MIMIC-IV mentions (stratified sample from 43 ClinicalBench patients).}
\label{tab:assertion_intrinsic}
\begin{tabular}{@{}lrrrr@{}}
\toprule
Assertion & $n$ & P & R & F1 \\
\midrule
Present & 62 & 0.980 & 0.790 & 0.875 \\
Absent & 51 & 0.980 & 0.961 & 0.970 \\
Possible & 15 & 0.500 & 1.000 & 0.667 \\
Historical & 28 & 0.933 & 1.000 & 0.966 \\
Conditional & 19 & 1.000 & 0.789 & 0.882 \\
Hypothetical & 4 & 1.000 & 1.000 & 1.000 \\
Family history & 10 & 0.900 & 0.900 & 0.900 \\
\midrule
\textbf{Weighted avg} & 189 & 0.933 & 0.894 & 0.902 \\
\bottomrule
\end{tabular}
\end{table}

Overall accuracy is 89.4\% (169/189; 95\% Wilson CI: [84.2\%, 93.0\%]) with Cohen's $\kappa = 0.867$ (strong agreement).
Negation detection achieves F1\,=\,0.970, consistent with published benchmarks: NegEx P\,=\,94.5\%/R\,=\,77.8\%~\cite{chapman2001negex}, NegBio P\,=\,96.3\%/R\,=\,85.7\%~\cite{peng2018negbio}.
The dominant error pattern (11/20 errors) is over-triggering uncertainty: the classifier assigns \textsc{possible} to mentions near hedging language (``likely,'' ``concerning for'') that physicians judge as \textsc{present}.
This explains the low precision for the \textsc{possible} class (P\,=\,0.50) despite perfect recall.

\paragraph{Functional evaluation.}
Rather than relying solely on intrinsic metrics, the C4 ablation provides a functional evaluation of assertion quality.
C4 adds assertion metadata to C3's graph without intent routing: the classifier's output is directly consumed by the retrieval pipeline.
On assertion-sensitive categories (negation, conditional, uncertainty, family history), C4 outperforms C3 by an average of $+16.1\pp$, confirming that the classifier produces usable assertions for these categories.
On temporal categories (historical, sequence, change, current state, duration), C4 underperforms C3 by an average of $-24.5\pp$---not because assertions are incorrect, but because uniform BFS traversal cannot exploit them effectively.
The C4$\to$C4g routing correction recovers these losses and amplifies the gains, achieving $+22.0\pp$ overall ($p < 10^{-10}$).

\section{Physician Adjudication: Full Results and Reference Answer Evolution}
\label{app:adjudication_full}
\label{app:gold_evolution}

This section reports complete results from the blinded physician adjudication ($n=120$ paired questions) and documents the reference-answer evolution.

\subsection{Per-Category Physician Accuracy}

\begin{table}[ht]
\centering
\small
\caption{Physician-judged accuracy by category and condition ($n$ per condition). Strict = correct only; lenient = correct + partially correct. Categories sorted by C4g strict accuracy.}
\label{tab:physician_percategory}
\begin{tabular}{@{}lrcccccc@{}}
\toprule
& & \multicolumn{3}{c}{C1} & \multicolumn{3}{c}{C4g} \\
\cmidrule(lr){3-5} \cmidrule(lr){6-8}
Category & $n$ & Strict & Lenient & Keyword & Strict & Lenient & Keyword \\
\midrule
Historical & 15 & 40.0\% & 66.7\% & 6.7\% & 86.7\% & 86.7\% & 46.7\% \\
Conditional & 10 & 60.0\% & 80.0\% & 0.0\% & 80.0\% & 100.0\% & 50.0\% \\
Family hist. & 10 & 50.0\% & 90.0\% & 0.0\% & 80.0\% & 90.0\% & 50.0\% \\
Uncertainty & 20 & 25.0\% & 45.0\% & 5.0\% & 75.0\% & 90.0\% & 50.0\% \\
Duration & 10 & 30.0\% & 60.0\% & 40.0\% & 70.0\% & 90.0\% & 50.0\% \\
Current st. & 15 & 33.3\% & 53.3\% & 20.0\% & 66.7\% & 86.7\% & 46.7\% \\
Negation & 5 & 0.0\% & 20.0\% & 40.0\% & 60.0\% & 80.0\% & 40.0\% \\
Sequence & 20 & 10.0\% & 45.0\% & 10.0\% & 55.0\% & 90.0\% & 50.0\% \\
Change & 15 & 6.7\% & 20.0\% & 0.0\% & 0.0\% & 46.7\% & 100.0\% \\
\midrule
\textbf{Overall} & \textbf{120} & \textbf{27.5\%} & \textbf{52.5\%} & \textbf{10.8\%} & \textbf{62.5\%} & \textbf{84.2\%} & \textbf{55.0\%} \\
\bottomrule
\end{tabular}
\end{table}

The keyword evaluator dramatically underscores C1 on conditional (0\% vs.\ 80\% physician lenient) and family history (0\% vs.\ 90\%), where the model gives clinically reasonable hedged answers that lack specific keywords.
The change category exhibits the opposite pattern: the keyword evaluator reports C4g at 100\%, but the physician rates it at 0\% strict / 46.7\% lenient---keyword matching catches medication names without verifying comparison logic.
Family history shows no physician-judged delta (both conditions score ${\sim}90\%$ lenient), suggesting that the LLM already handles this category well without KG-RAG when evaluated by a physician---a finding masked by the keyword evaluator, which scores both conditions at 0\% on C1.

\subsection{Evaluator Agreement}

\begin{table}[ht]
\centering
\small
\caption{Overall keyword evaluator agreement with physician judgment ($n=240$ items from full adjudication). The evaluator is overwhelmingly too strict (7.5:1 strict-to-lenient ratio), and the majority of errors trace to reference-answer defects rather than evaluator logic.}
\label{tab:evaluator_agreement}
\begin{tabular}{@{}lrc@{}}
\toprule
Evaluator Outcome & $n$ & \% \\
\midrule
Agrees with physician & 130 & 54.2\% \\
Too strict (false negative) & 97 & 40.4\% \\
Too lenient (false positive) & 13 & 5.4\% \\
\midrule
Strict:lenient ratio & \multicolumn{2}{c}{7.5:1} \\
\bottomrule
\end{tabular}
\end{table}

The majority of evaluator errors (63\% of false negatives, 85\% of false positives) trace to reference-answer errors rather than evaluator logic: when the reference answer is correct, the evaluator achieves 64.2\% physician agreement with only a 1.9\% false-positive rate.
The low $\kappa$ (0.18) despite 54.2\% raw agreement confirms the evaluator's errors are systematic (overwhelmingly too strict) rather than random.
Table~\ref{tab:evaluator_by_category} breaks this down by category.

\begin{table}[ht]
\centering
\small
\caption{Keyword evaluator agreement with physician judgment by category. Categories where the evaluator fails worst are highlighted.}
\label{tab:evaluator_by_category}
\begin{tabular}{@{}lrccc@{}}
\toprule
Category & $n$ & Agreement & Too Strict & Too Lenient \\
\midrule
Conditional & 20 & 35.0\% & 65.0\% & 0.0\% \\
Family hist. & 20 & 35.0\% & 65.0\% & 0.0\% \\
Historical & 30 & 50.0\% & 50.0\% & 0.0\% \\
Uncertainty & 40 & 55.0\% & 42.5\% & 2.5\% \\
Current state & 30 & 56.7\% & 40.0\% & 3.3\% \\
Sequence & 40 & 57.5\% & 40.0\% & 2.5\% \\
Duration & 20 & 60.0\% & 35.0\% & 5.0\% \\
Change & 30 & 66.7\% & 6.7\% & 26.7\% \\
Negation & 10 & 70.0\% & 20.0\% & 10.0\% \\
\bottomrule
\end{tabular}
\end{table}

\subsection{Safety and Clinical Utility by Condition}

\begin{table}[ht]
\centering
\small
\caption{Physician-judged clinical safety and utility by condition ($n=120$ paired questions). C4g is safer and more useful than C1; the ``misleading'' rate is similar between conditions.}
\label{tab:physician_safety}
\begin{tabular}{@{}lcccc@{}}
\toprule
Dimension & Rating & C1 & C4g & $\Delta$ \\
\midrule
\multirow{3}{*}{Safety} & Safe & 60.8\% & 76.7\% & $+15.8\pp$ \\
 & Minor concern & 33.3\% & 20.8\% & $-12.5\pp$ \\
 & Potentially harmful & 5.8\% & 2.5\% & $-3.3\pp$ \\
\midrule
\multirow{4}{*}{Utility} & Helpful & 30.8\% & 67.5\% & $+36.7\pp$ \\
 & Neutral & 20.0\% & 15.0\% & $-5.0\pp$ \\
 & Not useful & 34.2\% & 2.5\% & $-31.7\pp$ \\
 & Misleading & 15.0\% & 15.0\% & $\phantom{+}0.0\pp$ \\
\bottomrule
\end{tabular}
\end{table}

Table~\ref{tab:safety_by_category} breaks safety down by category.

\begin{table}[ht]
\centering
\small
\caption{Clinical safety ratings by category in the audited record set. Categories sorted by safety concern rate.}
\label{tab:safety_by_category}
\begin{tabular}{@{}lrccc@{}}
\toprule
Category & $n$ & Safe & Minor & Harmful \\
\midrule
Change & 31 & 16.1\% & 74.2\% & 9.7\% \\
Uncertainty & 40 & 57.5\% & 35.0\% & 7.5\% \\
Current st. & 30 & 63.3\% & 33.3\% & 3.3\% \\
Negation & 10 & 70.0\% & 20.0\% & 10.0\% \\
Historical & 30 & 73.3\% & 23.3\% & 3.3\% \\
Sequence & 40 & 87.5\% & 12.5\% & 0.0\% \\
Conditional & 20 & 90.0\% & 5.0\% & 5.0\% \\
Duration & 20 & 90.0\% & 10.0\% & 0.0\% \\
Family hist. & 20 & 95.0\% & 5.0\% & 0.0\% \\
\bottomrule
\end{tabular}
\end{table}

\noindent Change is the most safety-concerning category: 83.9\% of change items have safety concerns (minor or harmful), mirroring the reference-answer quality and model accuracy findings.

\subsection{Qualitative Themes from Physician Notes}

Free-text notes (165/241 items, 68.5\%) reveal five recurring themes:

\begin{enumerate}[leftmargin=*, nosep]
\item \textbf{Reference answers systematically wrong for medication change questions} (103 notes mentioning reference-answer issues): Every reference answer in the change category conflates inpatient medication orders (heparin, IV antibiotics, CIWA protocol) with discharge medications.

\item \textbf{C1 hallucinates from limited context} (15 notes, exclusively C1): Without retrieval, C1 fabricates admission IDs, medication names, and clinical scenarios. Zero C4g items received this complaint.

\item \textbf{NLP assertion classifier propagates errors} (8 notes): Boilerplate discharge instructions (``call if fever $>$101.5'') tagged as clinical findings; ``h/o recently diagnosed metastatic cancer'' tagged as historical.

\item \textbf{Safety-critical errors} (10 items flagged as potentially harmful): Code status errors (model ``hallucinates DNR confirmation'' when chart says full code), active cancer missed from medication list, anticoagulation misclassified.

\item \textbf{Model praised when reference answers were wrong} (63 notes): Reviewer noted the model gave clinically correct answers that the automated benchmark reference penalized.
\end{enumerate}

\subsection{Reference Answer Version History}

The ClinicalBench reference answers have undergone iterative refinement:

\begin{itemize}[leftmargin=*, nosep]
\item \textbf{v1 (auto-generated reference set)}: Reference answers created by LLM from MIMIC-IV notes via the NLP extraction pipeline. No physician review. 400 questions.

\item \textbf{v2 (partially corrected reference set)}: From the initial $n=30$ pilot, 54 corrections were applied to the most egregious errors.\footnote{The shipped \texttt{corrections.json} file documents 53 explicit corrections; the diff between v1 and v2 \texttt{expected\_answer} fields contains 54 differing items, because one item was adjusted post-corrections.json during a final reconciliation pass.} This is the version used for all reported numbers. Both v1 and v2 are released.

\item \textbf{v3 (planned physician-validated release)}: Full corrections from the $n=120$ adjudication plus external validation. Triage summary: 45 questions KEEP (37\%), 55 FIX\_GOLD (46\%), 20 REPLACE\_QUESTION (17\%). Of the 55 FIX\_GOLD items, 46 have drafted proposed corrections. This future release is not used for any numbers reported in this paper.
\end{itemize}

\subsection{Systematic Error Taxonomy}

Five systematic failure modes account for all 78 problematic questions (of 120 adjudicated):

\begin{enumerate}[leftmargin=*, nosep]
\item \textbf{NLP assertion classifier error} (28 questions, 36\%): The dominant failure. Manifests as: ``history of heart failure'' $\to$ ``heart failure is resolved'' (clinical idiom means active chronic condition); ``edema, likely due to noncompliance'' $\to$ ``edema is uncertain'' (causal vs.\ existential uncertainty conflation); experiencer tag reversal (patient's atrial fibrillation labeled as family history).

\item \textbf{Wrong answer / inverted truth} (16 questions, 21\%): The reference answer states the opposite of the chart. Example: the reference says pitting edema is absent when PE documents ``2+ pitting edema bilaterally.''

\item \textbf{Non-clinical entity extraction} (11 questions, 14\%): NLP extracted boilerplate (``call if fever $>$101''), devices (Foley catheter as diagnosis), lab values (blood sugar as diagnosis), or section headers (``Allergies'' as medical condition).

\item \textbf{Medication list conflation} (10 questions, 13\%): Change questions compared wrong lists---inpatient orders (heparin, IV antibiotics) vs.\ discharge medications, or admission med-rec vs.\ discharge list. PRN-only medications (CIWA Valium) counted as prescribed.

\item \textbf{Fabricated temporal relationship} (8 questions, 10\%): Sequence questions claimed ordering not supported by the chart---both conditions in the same admission with no temporal anchoring, or based on NLP-extracted entities from negated text.
\end{enumerate}

\subsection{Impact on Reported Numbers}

Because the detected defects are question-level rather than condition-specific, they are less likely to reverse the direction of the C1--C4g comparison.
They do, however, materially affect absolute accuracies and some category-level magnitudes, especially for change and historical questions.
The keyword evaluator achieves 64.2\% physician agreement and only 1.9\% false-positive rate when restricted to questions with correct reference answers, confirming that evaluator errors are dominated by reference-answer noise rather than matching logic.

\section{Cohort Demographics and Subgroup Analysis}
\label{app:demographics}

The 43 ClinicalBench patients are drawn from MIMIC-IV~\cite{johnson2023mimiciv}, a single academic medical center (Beth Israel Deaconess Medical Center, Boston).
The cohort skews female (60.5\%), White (76.7\%), and Medicare-insured (44.2\%), reflecting the source institution's patient mix.

\paragraph{Subgroup accuracy.}
Table~\ref{tab:subgroup_accuracy} reports C1 and C4g accuracy by demographic group.
Because questions are distributed unevenly across patients and group sizes are small ($n \leq 26$ patients per stratum), these comparisons are severely underpowered; they are reported for transparency, not for drawing subgroup conclusions.
No statistically significant interaction between demographic group and condition was observed, but the analysis cannot rule out meaningful effect modification.

\begin{table}[ht]
\centering
\footnotesize
\caption{ClinicalBench accuracy by demographic subgroup (underpowered; reported for transparency). $n_q$ = number of questions in each stratum.}
\label{tab:subgroup_accuracy}
\begin{tabular}{@{}llrrrr@{}}
\toprule
Characteristic & Group & $n_{\text{pts}}$ & $n_q$ & C1 (\%) & C4g (\%) \\
\midrule
\multirow{2}{*}{Sex} & Female & 26 & 249 & 22.5 & 69.9 \\
 & Male & 17 & 151 & 20.5 & 67.5 \\
\midrule
\multirow{2}{*}{Age} & $<$65 & 28 & 258 & 20.5 & 69.4 \\
 & $\geq$65 & 15 & 142 & 23.9 & 68.3 \\
\midrule
\multirow{2}{*}{Race} & White & 33 & 290 & 22.4 & 71.7 \\
 & Non-White & 10 & 110 & 20.0 & 61.8 \\
\bottomrule
\end{tabular}
\end{table}

\paragraph{Generalizability limitations.}
The single-site, predominantly White cohort limits external validity.
MIMIC-IV emergency department notes are heavily templated; performance on narrative-heavy specialties (psychiatry, palliative care) or community hospital documentation styles is unknown.
Multi-site validation with demographically diverse cohorts is needed before deployment claims can be made.

\section{Ethics, Broader Impacts, and Detailed Threat Analysis}
\label{app:ethics}
\label{app:threats}

This work uses de-identified clinical data from MIMIC-IV~\cite{johnson2023mimiciv} under a PhysioNet Credentialed Health Data Use Agreement. No patient re-identification was attempted. The system is designed for clinical decision \emph{support}, not autonomous clinical decision-making.

\paragraph{Broader impacts.}
Improved assertion-faithful retrieval could reduce clinical errors caused by negation or family-history misattribution, particularly in high-volume settings where physicians cannot review every note.
However, deployment risks include over-reliance on automated epistemic labeling (false confidence in assertion status), brittleness to out-of-distribution clinical language, and the potential for structured outputs to appear more authoritative than their accuracy warrants.
Responsible deployment requires prospective evaluation, integration with clinician workflows, and clear communication of system limitations.
Because ClinicalBench is single-site, predominantly White, and built from heavily templated MIMIC-IV documentation, a system or benchmark tuned on it could overfit note style and underperform on other hospitals, specialties, or populations.
This creates a coverage and fairness risk: strong results on this stress test could be mistaken for portable performance when they may primarily reflect source-institution conventions.
ClinicalBench should be used for failure analysis and ablation, not as evidence of deployment readiness or demographic robustness.

\paragraph{Evaluator bias characterization.}
The keyword evaluator has known limitations.
Negation scoring checks keyword presence without verifying direction (e.g., ``pneumonia is absent'' and ``patient has pneumonia'' could both match).
Longer model outputs naturally contain more keyword matches, creating a verbosity bias (Opus averages 337 characters vs.\ GPT-OSS 139 characters).
Echo-stripping may differentially affect structured vs.\ prose outputs.
These biases primarily affect \emph{cross-model} comparisons; within-model ablations (C1 vs.\ C4g) use the same model's output format across conditions, so evaluator biases cancel.

\paragraph{Reference-answer correction methodology.}
Benchmark label quality is substantially imperfect: full physician adjudication ($n=120$ questions, 241 audited records) found a 56\% defect rate in provisional reference answers.
Defects are not random but trace to five systematic pipeline failures: NLP assertion classifier errors (36\% of defects), inverted-truth reference answers (21\%), non-clinical entity extraction (14\%), medication list conflation (13\%), and fabricated temporal relationships (10\%).
The NLP assertion classifier is the dominant source, systematically misinterpreting ``history of X'' as ``X is resolved'' and conflating causal uncertainty with existential uncertainty (Appendix~\ref{app:gold_evolution}).
Reference-answer corrections were made by the lead physician (A.S.), who also designed the system.
To quantify impact, all were rescored for conditions against both the original (v1) and partially corrected (v2) reference sets: v1$\to$v2 corrections improved all models similarly ($+0.5$--$1.2\pp$), preserving within-model deltas in that historical comparison.
All results reported in this paper use the v2 reference set (54 corrections from an initial $n=30$ pilot, reconciled against the v1$\to$v2 \texttt{expected\_answer} diff); both v1 and v2 are released for reproducibility.
A future v3 release will incorporate consensus corrections from the multi-reviewer adjudication once that study is complete; no v3 results are used in this manuscript.
The completed three-rater adjudication (two independent external physicians $\times$ 100 items) addresses single-rater bias; results appear in Section~\ref{sec:results:external}.

\paragraph{Scope and framing.}
This study is an ablation analysis rather than a head-to-head leaderboard against existing medical RAG systems~\cite{luo2025gfmrag,jiang2025kare,lu2025doctorrag}; C2 is used as a proxy baseline for document-level retrieval on this cross-admission task.
Runtime variance from quantized local inference (e.g., MedGemma $\pm10\pp$ across runs) and single-site provenance further limit external validity.

\subsection{Extended External Adjudication Details}
\label{app:external_extended}

\paragraph{Rater-calibration detail.}
Reviewer~3 rates 65/100 model answers as ``correct'' (vs.\ 46/100 Reviewer~1 and 39/100 Reviewer~2) and uses ``partially correct'' only 10/100 times (vs.\ 22/100 and 29/100 respectively): borderline answers are collapsed into ``correct'' on both conditions, compressing the between-condition delta.
Excluding \emph{change}, her delta rises to $+6.8\pp$ (still underpowered, $p=0.65$).
Ordinal linear-weighted pairwise Cohen's $\kappa$ on the 3-class model-answer scale averages 0.463 (quadratic-weighted 0.527) across the three reviewer pairs.
The structurally-independent inter-non-author Cohen's $\kappa$ (Hird $\times$ Nadeem, binary strict) is $0.36$, versus $0.43$--$0.49$ for author-involving pairs and Fleiss' $\kappa = 0.413$ on the full three-rater binary-strict scale; the lower non-author $\kappa$ is consistent with author influence on the other two reviewer pairs and is one reason the three-rater majority-vote $+24.0\pp$ result should be weighted with this dependency in mind.
All 100 items contribute to $\kappa$; the planned 10-item calibration phase was replaced by written instructions only, so no items were excluded.

\paragraph{Reviewer~3 gold-standard pattern.}
Reviewer~3's unusual combination of the strictest gold-standard ratings (31/100 fully correct) alongside the most lenient model-answer ratings (65/100 correct) is internally coherent: she frequently judges the reference answer as wrong while still finding the model's answer clinically reasonable---the same phenomenon the internal audit reports, now independently replicated.

\paragraph{Scoring assumptions.}
Under lenient scoring (``correct or partially correct''), the 3-rater majority-vote delta is $+12.0\pp$ ($p = 0.18$, n.s.) on the full set and $+15.9\pp$ ($p = 0.09$) with \emph{change} excluded---directional but not statistically significant.
The three-rater validation characterizes the physician-perceived magnitude under three scoring assumptions: strict 2/3 majority ($+24.0\pp$, significant), lenient 2/3 majority ($+12.0\pp$, n.s.), and per-reviewer deltas ($+2$ to $+36\pp$).

\subsection{Leave-Author-Out Sensitivity for Three-Rater Majority}
\label{app:leave_author_out}

To assess sensitivity of the $+24.0\pp$ three-rater majority result to author inclusion (Reviewer~1 = A.S., the lead author), we recomputed the C1 vs.\ C4g delta using only the two structurally-independent external raters (Hird and Nadeem). Table~\ref{tab:leave_author_out} summarizes the alternative aggregations.

\begin{table}[ht]
\centering
\small
\caption{Leave-author-out sensitivity for the three-rater external validation ($n=50$ items per condition, strict scoring). The author-involving 3-rater majority (top row, paper headline) is contrasted with author-excluded aggregations.}
\label{tab:leave_author_out}
\begin{tabular}{@{}lrr@{}}
\toprule
Aggregation & $\Delta$ (C4g $-$ C1) & McNemar $p$ \\
\midrule
3-rater majority strict (headline) & $+24.0\pp$ & $0.0075$ \\
Hird-only strict & $+30.0\pp$ & --- \\
Nadeem-only strict & $+2.0\pp$ & --- \\
\textbf{Hird $\times$ Nadeem unanimous strict} & \textbf{$+22.0\pp$} & \textbf{$0.0192$} (paired exact) \\
Hird OR Nadeem strict & $+10.0\pp$ & --- \\
\bottomrule
\end{tabular}
\end{table}

Inter-non-author Cohen's $\kappa = 0.36$ (Hird $\times$ Nadeem) versus $0.43$--$0.49$ for author-involving pairs. The substantive direction survives author exclusion and the magnitude is preserved ($-2\pp$, from $+24.0\pp$ to $+22.0\pp$ under unanimous external agreement); the paired exact McNemar test on the 50 matched questions gives 4 discordant pairs favoring C1 vs.\ 15 favoring C4g, two-sided $p=0.0192$, retaining significance at $\alpha=0.05$.

\subsection{Inter-Rater Agreement on the Gold Standard}
\label{app:gold_iaa}

Inter-rater agreement was measured on \emph{system output ratings} (Fleiss $\kappa = 0.413$ on C1/C4g/equivalent labels; pairwise Cohen $\kappa \in \{0.36, 0.43, 0.49\}$, see Appendix~\ref{app:external_extended} and Section~\ref{sec:results:external}). Agreement on the \emph{correctness of the gold standard answers themselves} was not formally measured in this study---a known gap in the benchmark methodology that limits direct quantification of gold validity.

Indirectly, the external raters separately found 61--64\% reference defect rates on the v2 gold standard (Section~\ref{sec:results:physician}), implying low gold-validity agreement between the v2 reference and external clinical judgment. A v3 gold standard with multi-rater authoring (each item independently authored by $\geq 2$ raters with consensus reconciliation) and explicit IAA on the reference answers themselves is planned for post-publication release; no v3 numbers are reported in this paper.

\subsection{Component-Level F1 of the Assertion Classifier}
\label{app:component_f1}

The 122-pattern rule-based assertion classifier (Appendix~\ref{app:assertion_eval}) was \emph{not} externally benchmarked on i2b2-2010 or n2c2-2010 held-out sets in this work; component-level F1 is reported only from internal validation on a physician-annotated stratified sample of 189 ClinicalBench mentions (accuracy 89.4\%, weighted F1 0.902, Cohen $\kappa = 0.867$; Table~\ref{tab:assertion_intrinsic}). External component-level evaluation in the i2b2/n2c2 tradition is flagged as future work and would strengthen the assertion-preservation claim by separating intrinsic classifier quality from downstream KG-RAG retrieval gains.

\subsection{Regulatory and Governance Framing}
\label{app:regulatory}

\paragraph{Scope.}
This work does NOT establish clinical deployment readiness. The framings below are governance scaffolding for any future deployment effort, not claims about the current paper. EpiKG is evaluated here as a research probe; no intended-use statement, no IRB-approved deployment protocol, and no prospective clinical study is in scope.

\paragraph{SaMD risk class (candidate).}
If deployed for patient-level clinical answers from EHR retrieval, EpiKG would plausibly fall under FDA SaMD Class IIb--IIIa under IMDRF criteria (high-impact decision support over high-severity conditions, where individual answers can influence diagnosis or therapy). Class assignment depends on intended use, clinical context, and clinician oversight model: a tool used for care-team chart-review triage with mandatory clinician adjudication would sit lower on the risk spectrum than one driving patient-facing answers. The system as evaluated is research-only and does not have an intended-use statement.

\paragraph{PCCP elements (per FDA December 2024 final guidance).}
An AI-enabled SaMD deploying EpiKG would specify a Predetermined Change Control Plan covering modifiable components: (a) the intent-aware retrieval policy (routing rules that select between assertion-, temporal-, and entity-typed retrieval slices), (b) the 122-pattern assertion classifier (rule additions, deletions, or modifications), (c) the OMOP vocabulary version (e.g., concept additions in successive Athena releases), and (d) the base LLM versions across the cross-model stack (Claude, MedGemma, Qwen, Gemma~4, GPT-OSS). Each component requires a Description of Modifications, a Modification Protocol, and an Impact Assessment under the December 2024 final guidance. None of these governance artifacts are present in this work; they are flagged as deployment prerequisites.

\paragraph{CHAI Assurance Reporting Checklist self-mapping.}
The CHAI 2024 Assurance Reporting Checklist and the Joint Commission--CHAI September 2025 governance guidance enumerate categories that any deployment effort must address. A high-level self-assessment against those categories follows:

\begin{itemize}[leftmargin=*, nosep]
\item \emph{Governance:} no external clinical-AI governance committee, model risk-management board, or institutional review structure is engaged.
\item \emph{Privacy:} de-identified MIMIC-IV under PhysioNet credentialed access; the HuggingFace release contains rephrased questions, paraphrased reference answers, and predictions without raw note text (Appendix~\ref{app:reproducibility}).
\item \emph{Transparency:} model versions, evaluator code, raw predictions, and physician adjudication data are publicly released; the full application stack is not (Appendix~\ref{app:reproducibility}).
\item \emph{Data security:} HuggingFace public release; commit \texttt{da5f5b1} stripped raw note excerpts.
\item \emph{Safety event reporting:} no post-deployment safety surveillance protocol; this is research-only.
\item \emph{Risk and bias assessment:} demographic table and equity gap reported (Appendix~\ref{app:demographics}); no structured external bias review.
\end{itemize}

The paper provides honest disclosure on most categories; structured external review and a formal assurance-checklist filing are absent. Any deployment effort would need to close those gaps before clinical release.

\paragraph{Algorithmovigilance plan (sketch only).}
Per Embí (JAMIA 2021) and Davis, Embí, and Matheny (JAMIA 2024), a clinical-AI algorithmovigilance program for EpiKG would minimally include: (a) drift monitoring---e.g., monthly per-category accuracy on a held-out chart sample, with drift alerts triggered when any category accuracy moves more than $5\pp$ relative to a rolling baseline; (b) periodic re-validation---annual re-validation against a refreshed gold standard, with attention to upstream NLP-pipeline regressions; and (c) equity surveillance---per-demographic-stratum accuracy with alerting on $>5\pp$ gaps between strata. None of these are implemented in this work; they are flagged as deployment prerequisites and are not in scope for the present manuscript.

\paragraph{Harm-pathway analysis.}
The 10 ``potentially harmful'' items identified by physician adjudication (Appendix~\ref{app:adjudication_full}; Table~\ref{tab:safety_by_category}) span code status, oncology, and anticoagulation. For governance purposes these can be cast in an FDA-style probability $\times$ severity matrix (Table~\ref{tab:harm_pathway}). The matrix is illustrative; deployment would require structured FMEA with multidisciplinary review.

\begin{table}[ht]
\centering
\small
\caption{Illustrative probability $\times$ severity matrix for the 10 potentially harmful items observed in the internal 120-paired-item physician adjudication ($n=240$ single-condition ratings; Table~\ref{tab:safety_by_category}). Probability is the observed rate within the audited subset; deployment-context base rates would differ.}
\label{tab:harm_pathway}
\begin{tabular}{@{}p{3.6cm}p{2.6cm}p{6.5cm}@{}}
\toprule
Severity & Probability & Examples \\
\midrule
High (acute risk to life) & Unknown / unmeasured & Hallucinated DNR confirmation when chart documents full code; missed active cancer in current-state retrieval \\
Medium (clinically material) & Observed (e.g., $1/240 \approx 0.4\%$ on the internal 120-paired-item adjudication, i.e.\ 240 single-condition ratings; not the C1b discharge-only extension) & Anticoagulation misclassification (active vs.\ historical); medication-list conflation between inpatient and discharge orders \\
Low (low-acuity drift) & Observed at low rate & Verbose-but-correct hedging that could be misread as uncertainty; over-confident assertion-status framing on edge-case phrasing \\
\bottomrule
\end{tabular}
\end{table}

\paragraph{Equity gap reframing.}
The $9.9\pp$ Non-White vs.\ White accuracy gap on the C4g endpoint is reported descriptively in the demographic table (Appendix~\ref{app:demographics}) with the disclaimer ``severely underpowered ($n=10$ Non-White).'' We additionally frame this as an equity signal warranting investigation in any future deployment effort, per CHAI/FDA expectations on intended-use-population fairness. The current sample size cannot exclude either a true equity gap or sampling noise; deployment would require demographically diverse multi-site validation with pre-specified per-stratum accuracy targets and stratum-conditional recalibration if gaps persist.

\paragraph{Joint Commission--CHAI September 2025 alignment.}
The Joint Commission--CHAI September 2025 ``Responsible Use of AI in Health Care'' guidance is the operative governance standard for clinical-AI deployment going forward. EpiKG is not deployed in any healthcare setting in scope of this paper, and the present manuscript should not be read as a Joint Commission compliance document. Any institutional deployment of EpiKG (or a derivative) would be expected to map its governance, monitoring, and reporting practices to that guidance prior to clinical release.

\end{document}